\documentclass{article} 
\usepackage{iclr2025_conference,times}

\usepackage{amsmath,amsfonts,bm}

\def\1{\bm{1}}

\DeclareMathAlphabet{\mathsfit}{\encodingdefault}{\sfdefault}{m}{sl}
\SetMathAlphabet{\mathsfit}{bold}{\encodingdefault}{\sfdefault}{bx}{n}

\newcommand{\E}{\mathbb{E}}

\DeclareMathOperator*{\argmax}{arg\,max}
\DeclareMathOperator*{\argmin}{arg\,min}

\usepackage{hyperref}
\usepackage{url}
\usepackage{amsmath}
\usepackage{amsthm}
\usepackage{amssymb}
\usepackage{bm}
\usepackage{hyperref}
\usepackage{algorithm}
\usepackage{algpseudocode}
\usepackage{enumitem}
\usepackage{subcaption}
\usepackage{multirow}
\usepackage{bbm}
\usepackage{booktabs}
\usepackage{graphicx}
\usepackage{makecell}
\usepackage{array}

\newtheorem{theorem}{Theorem}[section]

\newtheorem{lemma}[theorem]{Lemma}
\newtheorem{definition}{Definition}
\newtheorem{assumption}{Assumption}

\newcommand{\V}[2]{V^{#1}_{#2}}

\newcommand{\Q}[2]{Q^{#1}_{#2}}

\newcommand{\Qbar}[2]{\bar{Q}^{#1}_{#2}}

\newcommand{\pstar}[0]{P^{\star}}
\newcommand{\phat}[0]{\hat{P}}
\newcommand{\pistar}[0]{\pi^{\star}}
\newcommand{\piref}[0]{\mu}

\newcommand{\Sspace}[0]{\mathcal{S}}
\newcommand{\Aspace}[0]{\mathcal{A}}
\newcommand{\Real}[0]{\mathbb{R}}
\newcommand{\Rset}[0]{\mathcal{R}}
\newcommand{\Pset}[0]{\mathcal{P}}
\newcommand{\Fset}[0]{\mathcal{F}}

\newcommand{\rhat}[0]{\hat{r}}
\newcommand{\rstar}[0]{r^{\star}}

\newcommand{\dataPref}[0]{\mathcal{D}_{\text{pref}}}
\newcommand{\dataTraj}[0]{\mathcal{D}_{\text{traj}}}
\newcommand{\dataRollout}[0]{\mathcal{D}_{\text{rollout}}}

\newcommand{\lossOpt}[0]{\mathcal{L}_{\text{opt}}}

\newcommand{\emplossReward}[0]{\hat{\mathcal{L}}_{R}}
\newcommand{\emplossDev}[0]{\hat{\mathcal{E}}}
\newcommand{\emplossTran}[0]{\hat{\mathcal{L}}_{T}}
\newcommand{\emplossOpt}[0]{\hat{\mathcal{L}}_{\text{opt}}}

\newcommand{\Ctr}[0]{C_{\text{TR}}}
\newcommand{\Cstep}[0]{C_{\text{ST}}}

\newcommand{\norm}[2]{\left\|#1\right\|_{#2}}
\newcommand{\kldiv}[2]{D_{KL}\left( #1 \| #2 \right)}

\title{Adversarial Policy Optimization for Offline\\Preference-based Reinforcement Learning}

\author{Hyungkyu Kang \\
Seoul National University \\
Seoul, South Korea \\
\texttt{hyungkyu0119@snu.ac.kr} \\
\And
Min-hwan Oh \\
Seoul National University \\
Seoul, South Korea \\
\texttt{minoh@snu.ac.kr}
}

\iclrfinalcopy 
\begin{document}

\maketitle

\begin{abstract}
    In this paper, we study offline preference-based reinforcement learning (PbRL), where learning is based on pre-collected preference feedback over pairs of trajectories. While offline PbRL has demonstrated remarkable empirical success, existing theoretical approaches face challenges in ensuring conservatism under uncertainty, requiring computationally intractable confidence set constructions. We address this limitation by proposing Adversarial Preference-based Policy Optimization (\texttt{APPO}), a computationally efficient algorithm for offline PbRL that guarantees sample complexity bounds without relying on explicit confidence sets. By framing PbRL as a two-player game between a policy and a model, our approach enforces conservatism in a tractable manner. Using standard assumptions on function approximation and bounded trajectory concentrability, we derive a sample complexity bound. To our knowledge, \texttt{APPO} is the first offline PbRL algorithm to offer both statistical efficiency and practical applicability. Experimental results on continuous control tasks demonstrate that \texttt{APPO} effectively learns from complex datasets, showing comparable performance with existing state-of-the-art methods.
\end{abstract}

\section{Introduction}

While Reinforcement Learning (RL) has achieved remarkable success in real-world applications~\citep{mnih2013playing,silver2017mastering,kalashnikov2018scalable,brohan2022rt}, its performance heavily depends on the design of the reward function~\citep{wirth2017survey}, which can be challenging in practice. 
To address this issue, preference-based reinforcement learning (PbRL), 
also known as reinforcement learning with human feedback, has gained increasing attention as an alternative to manually designed rewards. 
In PbRL, a reward model is learned from preference feedback provided by human experts, who compare pairs of trajectories~\citep{christiano2017deep}. 
This approach enables the learning process to align better with human intentions. PbRL has demonstrated its effectiveness in various domains, including gaming~\citep{macglashan2017interactive,christiano2017deep,warnell2018deep}, natural language processing~\citep{ziegler2019fine,stiennon2020learning,nakano2021webgpt,ouyang2022training,bai2022training}, and robotics~\citep{brown2019extrapolating,shin2023benchmarks}.

However, collecting preference feedback can be costly, especially when real-time feedback from human experts is required. In such cases, learning from pre-collected data is preferred over online learning. This approach is referred to as \emph{offline} PbRL, where the learning process relies solely on pre-collected trajectories and preference feedback. Empirical studies have shown the effectiveness of offline PbRL ~\citep{kim2023preference,an2023direct,shin2023benchmarks,hejna2024inverse}, leveraging techniques from deep RL literature. On the theoretical side, prior works prove that trajectory concentrability with respect to the data-collecting distribution leads to a sample complexity bound~\citep{zhu2023principled,zhan2024provable,pace2024preference}. However, they rely on the explicit construction of confidence sets to achieve conservatism (pessimism). Dealing with such confidence sets in the general function approximation setting requires intractable optimizations: \citet{zhan2024provable} involve tri-level constrained optimization with respect to the confidence sets of rewards and transitions, \citet{pace2024preference} use uncertainty penalty defined as the width of confidence sets, and the analysis of \citet{zhu2023principled} is restricted to linear models. Despite provable sample complexity bounds, existing offline PbRL algorithms become computationally intractable with general function approximation.

In this work, we propose a computationally and statistically efficient offline PbRL algorithm,  \textit{Adversarial Preference-based Policy Optimization} (\texttt{APPO}). 
Our analysis is based on general function approximation for both the model and value function classes. Moreover, standard assumptions on function classes and bounded trajectory concentrability~\citep{zhan2024provable} are sufficient to establish our sample complexity bound. 
Beyond its strong statistical guarantees, our algorithm is simple to implement using standard optimization techniques.
The core idea behind our algorithm is the two-player game formulation of model-based PbRL, which has been used in other areas of RL~\citep{rajeswaran2020game,rigter2022rambo,cheng2022adversarially,shen2024principled,bhardwaj2024adversarial}. 
By framing PbRL as a game between a policy and a model,
we ensure conservatism without explicitly constructing intractable confidence sets. 
Furthermore, our novel reparameterization technique allows us to find a near-optimal policy efficiently via adversarial training. To the best of our knowledge, \texttt{APPO} is the first offline PbRL algorithm with both statistical performance guarantees and a practical implementation.
Our contributions can be summarized as follows:

\begin{itemize}
    \item We propose \texttt{APPO}, a simple algorithm for offline PbRL with general function approximation. Based on the two-player game formulation of PbRL in conjunction with our reparameterization technique for the reward model, our algorithm ensures provable conservatism without explicit construction of confidence sets. To our knowledge, our \texttt{APPO} is the first computationally efficient offline PbRL algorithm providing a sample complexity bound.
    \item We prove the sample complexity of our proposed algorithm under standard assumptions on the function classes and concentrability. The result is rooted in our novel sub-optimality decomposition, which shows that adversarial training leads to model conservatism.
    \item We present a practical implementation of \texttt{APPO} that can learn with large datasets using neural networks. Experiments on continuous control tasks demonstrate that \texttt{APPO} achieves performance comparable to existing state-of-the-art algorithms.
\end{itemize}

\subsection{Related Work} \label{sec:related work}

\textbf{Provable Online PbRL.}
In the tabular setting, \citet{novoseller2020dueling} developed an algorithm grounded in posterior sampling and the dueling bandit framework~\citep{yue2012k}, demonstrating an asymptotic rate for Bayesian regret.
\citet{xu2020preference} proposed an algorithm leveraging an exploration bonus for previously unseen states, which provides a sample complexity bound.
\citet{saha2023dueling} and \citet{zhan2024provablerewardagnostic} focused on the linear preference model with a known linear feature map, each offering regret and sample complexity bounds. However, their algorithms require solving an optimization $\argmax_{\pi, \pi'} \norm{\mathbb{E}_{\tau\sim\pi}[\phi(\tau)] - \mathbb{E}_{\tau\sim\pi'}[\phi(\tau)]}{\Sigma}$ for some positive definite matrix $\Sigma$, which is computationally intractable.
To address this challenge in the linear model, \citet{wu2024making} devised a randomized algorithm with a provable regret bound and further proposed a model-based posterior sampling algorithm under the bounded Eluder dimension~\citep{russo2013eluder} assumption, ensuring bounded Bayesian regret.
Recent works have also explored provably efficient algorithms under the general function approximation setting~\citep{chen2022human,wu2024making,chen2023provably}.
\citet{chen2022human} introduced an exploration-bonus-based algorithm that provides bounded regret in both pairwise and n-wise comparison settings.
Additionally, \citet{chen2023provably} leveraged the conditional value-at-risk operator~\citep{artzner1997thinking} to devise an algorithm with a regret guarantee. \citet{du2024explorationdriven} took a different approach, studying neural function approximation in the context of reward models. In another notable work, \citet{swamy2024a} reframed PbRL as a zero-sum game between two policies, encompassing general reward models.

\textbf{Provable Offline PbRL.}
While there has been a growing amount of research on online PbRL, the theoretical understanding of \mbox{offline} PbRL remains relatively limited. A primary challenge in offline PbRL, much like in offline standard RL, is ensuring sufficient conservatism in the model. \citet{zhu2023principled} addressed this \mbox{challenge} by proposing a pessimistic maximum likelihood estimation (MLE) algorithm for the linear model with known transitions.
\citet{zhan2024provable} \mbox{extended} this idea to general function approximation, highlighting the importance of trajectory concentrability in establishing a lower bound for sample complexity. Despite the provable sample complexity bound, their proposed algorithm, FREEHAND-transition, relies on solving $\argmax_{\pi} \argmin_{r\in\hat{\Rset}} \argmin_{P\in\hat{\Pset}} \{ \mathbb{E}_{\tau \sim P, \pi}[r(\tau)] - \mathbb{E}_{\tau\sim \pstar, \pi}[r(\tau)] \}$ where $\hat{\Rset}$ is the confidence set of rewards and $\hat{\Pset}$ is the confidence set of transitions, which is intractable in practice. 
\citet{pace2024preference} study a different but related setting, where the agent elicits high-quality preference information from offline data. Their method achieves conservatism through explicit uncertainty penalties defined as $u_{R}(\tau) = \sup_{r_1, r_2 \in\hat{\Rset}} |r_1(\tau) - r_2(\tau)|$ (reward uncertainty) and $u_{P}(s,a) = \sup_{P_1, P_2 \in\hat{\Pset}} \norm{P_1(\cdot\mid s,a) - P_2(\cdot\mid s,a)}{1}$ (transition uncertainty), which makes it intractable with general function approximation.
\citet{chang2024dataset} also explored a slightly different scenario where the data collection policy is known and online interaction is allowed. They demonstrated that a simple natural policy gradient combined with MLE reward is provably efficient, but their sample complexity bound is affected by an additional concentrability coefficient relative to KL-regularized policies.

\textbf{Adversarial Training in RL.}
Adversarial training is a widely used approach in RL literature~\citep{rajeswaran2020game,pasztor2024bandits}, especially offline (standard) RL~\citep{rigter2022rambo,cheng2022adversarially,bhardwaj2024adversarial}. The basic idea is to leverage adversarial training to implement conservative policy optimization. Recently, adversarial training has also been applied in human preference alignment~\citep{makar-limanov2024starlhf,cheng2024adversarial,shen2024principled}. The most closely related work to ours is \citet{shen2024principled}, which also formulated PbRL as a two-player game. However, their focus is on online PbRL, and while they provide proof of convergence for the optimization objective, this does not necessarily translate into a sample complexity guarantee.

\section{Preliminaries}

\textbf{Markov Decision Processes.} We consider an episodic MDP $(\Sspace,\Aspace,H,\{\pstar_{h}\}^H_{h=1},\{\rstar_h\}^H_{h=1})$, where $\Sspace$ and $\Aspace$ are the state space and the action space, $H$ is the length of each episode, $\pstar=\{\pstar_h\}^H_{h=1}$ is the collection of transition probability distributions, and $\rstar=\{\rstar_h\}^H_{h=1}$ is the collection of reward functions. Each episode starts at some initial state $s_1$ without loss of generality\footnote{Our result easily extends to the general case with an initial distribution $\rho(\cdot)$. 
We can modify the MDP by setting a fixed initial state $s_1$ and $\mathbb{P}_1(\cdot\mid s_1,a) = \rho(\cdot)$ for all $a\in\Aspace.$}, 
and the episode ends after $H$ steps. For each step $h\in[H]$, 
the agent observes the state $s_h$, and then takes action $a_h$. 
The environment generates reward $\rstar_h(s_h,a_h)$ (note that, in the preference-based learning setting, rewards at each step are unobservable to the agent) and next state $s_{h+1}$ according to the transition probability $\pstar_h(\cdot\mid s_h,a_h)$.

The agent takes actions based on its policy $\pi=\{\pi_h\}_{h\in[H]}$, where $\pi_h(\cdot\mid s)$ is a probability distribution over $\Aspace$. The state-value function and the action-value function of policy $\pi$ with respect to reward $r = \{r_h\}^H_{h=1}$ are the expected sum of rewards up to termination, starting from $s_h=s$ and $(s_h,a_h)=(s,a)$ respectively, following the policy $\pi$. Formally, they are defined as
\begin{align*}
    \V{\pi}{h, r}(s) := \mathbb{E}_{\pi}\left[ \sum^H_{h'=h} r_h(s_{h'}, a_{h'}) \mid s_h=s \right], \,\, \Q{\pi}{h, r} := \mathbb{E}_{\pi}\left[ \sum^H_{h'=h} r_h(s_{h'}, a_{h'}) \mid s_h=s, a_h=a \right].
\end{align*}
To simplify the notation, for $g : \Sspace \mapsto \Real$, we use $Pg(s,a)$ to denote $\mathbb{E}_{s'\sim P(\cdot\mid s,a)}[g(s')]$. For any policy $\pi$ and reward $r$, the Bellman equation relates $Q^{\pi}$ to $V^{\pi}$ as 
\begin{align*}
    \Q{\pi}{h, r}(s,a) = r_h(s,a) + \pstar\V{\pi}{h+1, r}(s,a), \,\, \V{\pi}{h, r}(s) = \mathbb{E}_{a \sim \pi_h(\cdot\mid s)}[\Q{\pi}{h,r}(s,a)], \,\, \V{\pi}{H+1}(s)=0.
\end{align*}
Given a policy $\pi=\{\pi_h\}_{h\in[H]}$, we define the state visitation distribution as $d^{\pi}_h(s) := \mathbb{P}_{\pi}(s_h=s)$ where $\mathbb{P}_{\pi}$ is the probability distribution of trajectories $(s_1,a_1,\dots,s_H,a_H)$ when the agent uses policy $\pi$. We overload the notation to denote the state-action visitation distribution, $d^{\pi}_h(s,a) := \mathbb{P}_{\pi}(s_h=s, a_h=a)$. In addition, we denote the distribution of trajectories under $\pi$ by $d^{\pi}(\tau)$.

\textbf{Offline Preference-based Reinforcement Learning.}
We consider the offline PbRL problem, where the agent cannot observe the true reward $\rstar$ but only binary preference feedback over trajectory pairs. Specifically, we are given a preference dataset $\dataPref = \{(\tau^{m,0}, \tau^{m,1}, y^m)\}^M_{m=1}$ that consists of i.i.d. trajectory pairs $\tau^{m,i} = \{s^{m,i}_h, a^{m,i}_h\}^H_{h=1}$ ($i=0,1$) sampled by some reference policy $\piref$. For a monotonically increasing link function $\Phi : \Real \mapsto [0,1]$, we assume the preference feedback $y^m\in\{0,1\}$ is generated by the following preference model:
\begin{align*}
    \mathbb{P}(y=1 \mid \tau^{0}, \tau^{1}) = \mathbb{P}(\tau^1 \text{ is preferred over }\tau^0) = \Phi(\rstar(\tau^1) - \rstar(\tau^0))
\end{align*}
where we denote $\rstar(\tau) = \sum^H_{h=1} \rstar_h(s_h,a_h)$ for given trajectory $\tau = (s_1,a_1,\dots,s_H,a_H)$. Additionally, we assume that $\kappa = 1/(\inf_{x\in[-R,R]} \Phi'(x))$, where $R$ is a bound on trajectory returns, is finite. When $\Phi$ is set to be the sigmoid function $\sigma(x) = 1/(1+\exp(-x))$, we obtain the widely used Bradely-Terry-Luce model~\citep{bradley1952rank}.
In addition to the preference dataset, we have an unlabeled trajectory dataset $\dataTraj = \{(\tau^{n,0}, \tau^{n,1})\}^N_{n=1}$ where the trajectory pairs are sampled i.i.d. by executing the reference policy $\piref$. 
The agent's goal is to find an $\epsilon$-optimal policy $\hat{\pi}$ with respect to the optimal policy $\pistar$, which satisfies $\V{\pistar}{1, \rstar}(s_1) - \V{\hat{\pi}}{1, \rstar}(s_1) \leq \epsilon$.

\textbf{General Function Approximation.} 
We consider general function approximation for rewards and transitions: the function class of rewards $\Rset$ and the function class of transitions $\Pset$. We do not impose any specific structure on them, so $\Rset$ and $\Pset$ can contain expressive functions such as neural networks. Based on the function classes, we construct a reward model using maximum likelihood estimation $\rhat \in \argmin_{r\in\Rset^H} \emplossReward(r)$ where 
\begin{align*}
    \emplossReward(r) = - \underset{(\tau^0,\tau^1,y)\sim\dataPref}{\E}\left[ y\cdot \log \Phi(r(\tau^1)-r(\tau^0)) + (1-y)\cdot \log \Phi(r(\tau^0)-r(\tau^1)) \right].
\end{align*}
Similarly, we learn a transition model $\phat_h \in \argmin_{P\in\Pset} \emplossTran(P;h)$ for all $h\in[H]$, where
\begin{align*}
    \emplossTran(P;h) = \mathbb{E}_{(s_h,a_h,s_{h+1})\sim\dataTraj}\left[ \log P(s_{h+1} \mid s_h, a_h) \right]
\end{align*}

\textbf{Additional Notations.} We denote $[n] := \{1,2,\dots,n\}$ for $n\in\mathbb{N}$. For $x,y\in\mathbb{R}^d$, $\langle x,y\rangle$ denotes the inner product of $x$ and $y$. Given a function $f : \Sspace\times\Aspace \mapsto \Real$ and a policy $\pi$, we write $f\circ \pi(s) := \mathbb{E}_{a\sim \pi(\cdot\mid s)}[f(s,a)]$. For given dataset $\mathcal{D}$, we use $\mathbb{E}_{x\sim\mathcal{D}}[f(x)]$ to denote $\frac{1}{|\mathcal{D}|}\sum_{x\in\mathcal{D}} f(x)$.

\section{Algorithm}

\subsection{PbRL as a Two-player Game}

A previous study on model-based PbRL by \citet{zhan2024provable} proves that the following optimization problem yields a near-optimal policy $\hat{\pi}$, for an appropriately chosen constant $\zeta$:
\begin{align}
    \hat{\pi} \in \argmax_{\pi} \min_{r\in\hat{\Rset}} \left( \V{\pi}{1, r}(s_1) - \V{\piref}{1, r}(s_1) \right) \text{  where } 
    \hat{\Rset} = \big\{ r\in \Rset^H : \emplossReward(r) \leq \emplossReward(\rhat) + \zeta \big\}. \label{eqn:freehand}
\end{align}
The minimization with respect to reward model $r\in\hat{\Rset}$ ensures conservatism, which is essential for a provable guarantee. However, the constrained optimization is intractable with general function approximation. 
To address this challenge, we formulate the model-based PbRL problem as a two-player Stackelberg game~\citep{von2010market} between the policy and the reward:
\begin{align}
    \hat{\pi} &\in \argmax_{\pi} \left(\V{\pi}{1, r^{\pi}}(s_1) - \V{\piref}{1, r^{\pi}}(s_1) \right)\notag \\
    \text{ subject to } r^{\pi} &\in \argmin_{r\in\Rset^H} \left( \V{\pi}{1, r}(s_1) - \V{\piref}{1, r}(s_1) + \mathcal{E}(r; \rhat) \right) . \label{eqn:Stackelberg game}
\end{align}
Here, $\mathcal{E}(r;\rhat)$ is a loss function that penalizes $r$ if it deviates from $\rhat$. In the Stackelberg game formulation, the reward minimizes $\V{\pi}{1, r}(s_1) - \V{\piref}{1, r}(s_1)$, while the policy maximizes it. We can interpret this competition by viewing $\V{\pi}{1, r}(s_1) - \V{\piref}{1, r}(s_1)$ as the relative performance of $\pi$ compared to $\piref$ with respect to reward $r$. Intuitively, $\pi$ maximizes the cumulative reward $r^{\pi}$, as in the standard RL setup. However, $r^{\pi}$ minimizes the cumulative reward when evaluated under $\pi$. This competition facilitates conservatism and makes $\pi$ robust to model error.

Then, what loss function $\mathcal{E}$ leads to a provable bound? A naive choice might be $\emplossReward(r)- \emplossReward(\rhat)$, as it leads to the Lagrangian dual form of the optimization problem in \eqref{eqn:freehand}, while disregarding the Lagrangian multiplier. However, the loss $\mathcal{E}(r;\rhat) = \emplossReward(r)- \emplossReward(\rhat)$ does not guarantee statistical efficiency, because the Stackelberg game in \eqref{eqn:Stackelberg game} does not include the Lagrangian multiplier for the likelihood constraint. Instead, we propose the trajectory-pair $\ell_1$ loss:
\begin{align*}
    \mathcal{E}(r;\rhat) = \mathbb{E}_{\tau^0,\tau^1 \sim \piref}\left[ \left| \{r(\tau^0) - r(\tau^1)\} - \{\rhat(\tau^0) - \rhat(\tau^1)\} \right| \right],
\end{align*}
which leads to a provable guarantee (Theorem~\ref{thm:upper bound}). Intuitively, this loss measures the deviation of $r$ from $\rhat$ by evaluating the difference in total reward (return) between the two trajectories. Given the unlabeled trajectory dataset $\dataTraj$, we approximate $\mathcal{E}(r;\rhat)$ with its finite-sample version:
\begin{align*}
    \emplossDev_{\dataTraj}(r; \rhat) = \mathbb{E}_{(\tau^0,\tau^1)\sim\dataTraj}\left[ \left| \{r(\tau^0) - r(\tau^1)\} - \{\rhat(\tau^0) - \rhat(\tau^1)\} \right| \right].
\end{align*}
In the following two sections, we discuss how to implement the optimization in \eqref{eqn:Stackelberg game} in a sample-efficient manner.

\subsection{Adversarial Optimization for PbRL}
In this section, we present an algorithm, \texttt{APPO-rollout}, that serves as a building block of our main algorithm.
For \texttt{APPO-rollout}, we consider the setting where the transition $\pstar$ is known or where online interaction (without preference feedback) is possible. This is a temporary assumption, as our main algorithm (Algorithm~\ref{alg:APPO}) works with an unknown transition. 

Algorithm~\ref{alg:APPO rollout} presents the pseudo-code of \texttt{APPO-rollout}, which is based on the Stackelberg game formulation of PbRL that we discussed. 
Inspired by the adversarial training methods in offline 
RL under the standard setting~\citep{cheng2022adversarially,rigter2022rambo,bhardwaj2024adversarial}, we alternately optimize the policy and reward to solve the optimization problem in \eqref{eqn:Stackelberg game}.

\begin{algorithm}[t]
    \caption{Adversarial Preference-based Policy Optimization with Rollout (\texttt{APPO-rollout})} \label{alg:APPO rollout}
    \begin{algorithmic}[1]
        \State \textbf{Input:} Number of rollouts $K_1, K_2$, constant $\eta$, $\pi^1_h = \text{Unif}(\Aspace)$ for all $h\in[H]$
        \State Estimate $\rhat \in \argmin_{r\in\Rset^H} \emplossReward(r)$
        \For{$t=1,\cdots, T$}
            \State Execute $\pi^t$ to collect $K_1$ trajectories $\dataRollout^t$
            \State Optimize $r^t \in \argmin_{r\in\Rset^H} \left( \mathbb{E}_{\tau\sim\dataRollout^t}[ r(\tau) ] - \mathbb{E}_{\tau\sim\dataTraj}[ r(\tau) ] + \lambda \hat{\mathcal{E}}_{\dataTraj}(r;\rhat) \right)$
            \State Compute $\Qbar{t}{}$ via \texttt{PE}$(\piref, \pi^t, \rhat, K_2)$ in Algorithm~\ref{alg:PE oracle}
            \State Update policy $\pi^{t+1}_h(a\mid s) \propto \pi^t_h(a\mid s) \exp(\eta\Qbar{t}{h}(s,a)) $ for all $h\in[H]$
        \EndFor
        \State Return $\bar{\pi} = \frac{1}{T}\sum^T_{t=1}\pi_t$
    \end{algorithmic}
\end{algorithm}

\textbf{Reward Model Update for Provable Conservatism.}
The reward model update aims to solve the following optimization problem approximately:
\begin{align}
    \argmin_{r\in\Rset^H} \Big( \underset{\tau\sim\pi^t}{\mathbb{E}}[ r(\tau) ] - \underset{\tau\sim\piref}{\mathbb{E}}[ r(\tau) ] + \lambda \mathcal{E}(r; \rhat) \Big) = \argmin_{r\in\Rset^H} \left( \V{\pi^t}{1, r}(s_1) - \V{\piref}{1, r}(s_1) + \lambda \mathcal{E}(r; \rhat) \right),\label{eqn:reward model optimization}
\end{align}
which is the inner optimization in \eqref{eqn:Stackelberg game}. The expectations $\mathbb{E}_{\tau\sim\piref}[ r(\tau) ]$ and $\mathcal{E}(r; \rhat)$ are approximated using offline data $\dataTraj$. Also, we collect trajectories by executing $\pi^t$, to compute the finite-sample version of $\mathbb{E}_{\tau\sim\pi^t}[ r(\tau) ]$. Note that the trajectory rollout (Line 4) is possible since we assume a known transition $\pstar$ or access to online interaction.

\textbf{Policy Update.}
After optimizing $r^t$, we estimate the action-value function of $\pi^t$ with respect to $r^t$ using a policy evaluation subroutine \texttt{PE}, whose pseudo-code is provided in Algorithm~\ref{alg:PE oracle}. This subroutine computes an approximate value function $\Qbar{t}{}$ using Monte Carlo estimation, providing an error bound relative to the true value function $\Q{\pi^t}{r^t}$. The theoretical analysis of \texttt{PE} is presented in Appendix~\ref{sec:PE oracle}. With the estimated value function $\Qbar{t}{}$, we then proceed to update the policy using trust region policy optimization (TRPO) \citep{schulman2015trust} update.

\subsection{\texttt{APPO}: Reparameterized Algorithm for Unknown Transition} \label{sec:APPO}

\begin{algorithm}[t]
    \caption{Adversarial Preference-based Policy Optimization (\texttt{APPO})} \label{alg:APPO}
    \begin{algorithmic}[1]
        \State \textbf{Input:} Constant $\eta$, Initial policy $\pi^1_h = \text{Unif}(\Aspace)$ for all $h\in[H]$
        \State Estimate $\rhat \in \argmin_{r\in\Rset^H} \emplossReward(r)$, $\phat_h \in \argmin_{P\in\Pset} \emplossTran(P;h)$ for all $h\in[H]$
        \For{$t=1,\cdots, T$}
            \State $f^t \in \underset{f\in\Fset^H}{\argmin} \left( \sum^H_{h=1} \mathbb{E}_{(s_h,a_h)\sim\dataTraj}\left[ f_h\circ\pi^t_h (s_h) - f_h(s_h, a_h) \right] + \lambda \hat{\mathcal{E}}_{\dataTraj}(f; \phat,\rhat) \right)$
            \State Update policy $\pi^{t+1}_h(a\mid s) \propto \pi^t_h(a\mid s) \exp(\eta f^t_h(s,a))$ for $h\in[H]$
        \EndFor
        \State Return $\hat{\pi} = \frac{1}{T}\sum^T_{t=1}\pi^t$
    \end{algorithmic}
\end{algorithm}

In this section, we consider the setting where the transition $\pstar$ is unknown.
In Algorithm~\ref{alg:APPO rollout}, the information from the transition $\pstar$ is utilized in Line 4, where we collect on-policy trajectories to approximate $\mathbb{E}_{\tau\sim\pi^t}[r(\tau)]$.
Moreover, the policy evaluation step in Algorithm~\ref{alg:PE oracle} involves trajectory rollouts. To bypass such on-policy rollouts, we make the following observation: 
\begin{align}
    \mathbb{E}_{\tau\sim\pi^t}[r(\tau)] - \mathbb{E}_{\tau\sim\piref}[r(\tau)] &= \V{\pi^t}{1, r}(s_1) - \V{\piref}{1, r}(s_1) \notag \\
    &= \sum^H_{h=1} \mathbb{E}_{(s_h,a_h)\sim d^{\piref}_h}\left[ (\Q{\pi^t}{h,r}\circ \pi^t_h)(s_h) - \Q{\pi^t}{h,r}(s_h,a_h) \right] \label{eqn:performance difference}
\end{align}
which is due to the performance difference lemma (Lemma~\ref{lemma:performance difference lemma}).
Since the expectation on the right is taken with respect to $d^{\piref}_h$, the data-generating distribution of $\dataTraj$, we can approximate it using $\dataTraj$.
Furthermore, given the policy $\pi^t$, the Bellman equation implies a mapping between reward models and action-value functions.
Specifically, for given reward model $r = \{r_h\}^H_{h=1}$, we have the action-value function $\{\Q{\pi^t}{h,r}\}^H_{h=1}$.
Conversely, suppose that we have a function class $\mathcal{F}$ that contains every action-value function.
For $f = \{f_h\}^H_{h=1}\in\mathcal{F}^H$, we can construct the corresponding reward model satisfying the Bellman equation $f_h = r_h + \pstar_h (f_{h+1}\circ \pi^t_{h+1})$. Formally, we define the induced reward models:

\begin{definition}[Induced reward model]
    Given $f = \{f_h\}^H_{h=1} \in\mathcal{F}^H$, and a policy $\{\pi_h\}^H_{h=1}$, we define the induced reward model $r^{\pi}_{\pstar, f} = \{ r^{\pi}_{h, \pstar, f} \}^H_{h=1}$ where $r^{\pi}_{h, \pstar, f} = f_h - \pstar_h (f_{h+1} \circ \pi_{h+1})$ for $h\in[H]$ (we set $f_{H+1}=0$ by convention).
\end{definition}

Therefore, given reward model $r$ and action-value function $f$, we have that
\begin{align*}
    \Q{\pi}{h, r} = r_h + \pstar (\Q{\pi}{h+1, r}\circ \pi_{h+1}), \,\,\, f_h = r^{\pi}_{h,\pstar,r} + \pstar_h (f_{h+1}\circ \pi_h) \text{ for all  } h\in[H].
\end{align*}
Importantly, \emph{the mapping does not need to be bijective} for our theoretical analysis, as long as the Bellman equation holds.
Using this mapping in conjunction with our observation in \eqref{eqn:performance difference}, we reparameterize the optimization problem in \eqref{eqn:reward model optimization} as:
\begin{align}
    &\argmin_{f\in\mathcal{F}^H} \left( \sum^H_{h=1} \mathbb{E}_{(s_h,a_h)\sim d^{\piref}_h}\left[ (f_h\circ \pi^t_h)(s_h) - f_h(s_h,a_h) \right] + \lambda \mathcal{E}(f; \pstar, \rhat) \right) \label{eqn:reparameterized optimization}\\
    &\text{where } \mathcal{E}(f; \pstar, \rhat) = \mathbb{E}_{(\tau_0, \tau_1) \sim \piref}\left[ \left| \{ r^{\pi^t}_{\pstar,f}(\tau^0) - r^{\pi^t}_{\pstar,f}(\tau^1)\} - \{\rhat(\tau^0) - \rhat(\tau^1)\} \right| \right]. \notag
\end{align}

The offline dataset $\dataTraj$ is sufficient to approximate the optimization objective in \eqref{eqn:reparameterized optimization} with
\begin{align*}
    &\mathbb{E}_{(s_h,a_h)\sim\dataTraj}\left[ (f_h\circ\pi^t_h) (s_h) - f_h(s_h, a_h) \right] \approx \mathbb{E}_{(s_h,a_h)\sim d^{\piref}_h}\left[ (f_h\circ \pi^t_h)(s_h) - f_h(s_h,a_h) \right] \\
    &\hat{\mathcal{E}}_{\dataTraj}(f ; \phat, \rhat) := \mathbb{E}_{(\tau_0, \tau_1) \sim \dataTraj}\left[ \left| \{ r^{\pi^t}_{\phat,f}(\tau^0) - r^{\pi^t}_{\phat,f}(\tau^1)\} - \{\rhat(\tau^0) - \rhat(\tau^1)\} \right| \right] \approx \mathcal{E}(f; \pstar, \rhat),
\end{align*}
where we use the estimated transition model $\phat$ in place of $\pstar$.
Moreover, since we directly optimize the action-value function, a policy evaluation oracle is not required to update the policy.
Therefore, this reparameterization allows us to solve the optimization problem in \eqref{eqn:Stackelberg game} without access to the true transition $\pstar$ or a policy evaluation oracle. The complete pseudo-code is presented in Algorithm~\ref{alg:APPO}.

\textbf{Remark on Computational Complexity.}
The computational complexity of \texttt{APPO} is primarily determined by the value function optimization (Line 4) and the policy update (Line 5).
Although optimizing $f^t$ is generally a non-convex problem, it can be efficiently implemented using gradient-based methods when $\mathcal{F}$ is a class of neural networks.
For the policy update, it is known that $\pi^{t+1}_h(a\mid s) \propto \pi^t_h(a\mid s) \exp(\eta f^t_h(s,a))$ is derived from the TRPO objective~\citep{schulman2015trust,neu2017unified}:
\begin{align*}
    \pi^{t+1}_h \in \argmax_{\pi} \mathbb{E}_{s_h\sim d^{\pi^t}_h}\left[  f^t_h \circ \pi(s_{h})  -\eta^{-1} \kldiv{\pi(\cdot\mid s_{h})}{\pi^t_h (\cdot\mid s_{h})} \right],
\end{align*}
which is widely used in deep RL. As a result, the policy update is efficient within the deep learning framework.
In practice, other policy optimization techniques~\citep{schulman2017proximal,fujimoto2018addressing,haarnoja2018soft} can also be applied.
Overall, \texttt{APPO} relies on solving two standard non-convex optimizations to compute $f^t$ and $\pi^t$, both of which are practical to implement with neural function approximation.
This computational efficiency contrasts with that of existing offline PbRL algorithms, which require intractable optimization over confidence sets, as discussed in Section~\ref{sec:related work}.

\section{Theoretical Analysis} \label{sec:theory}

In this section, we present theoretical analyses of our proposed algorithm, $\texttt{APPO}$. We note that \texttt{APPO-rollout} also guarantees a sample complexity bound, which is presented in Appendix~\ref{sec:proof of APPO rollout}.

We assume the reward class $\Rset$ and the transition class $\Pset$ are realizable and rewards are bounded. These are standard assumptions~\citep{chen2023provably,zhan2024provable,pace2024preference}.

\begin{assumption}[Reward realizability] \label{assumption:reward function class}
    We have $\rstar_h \in \Rset$ for all $h\in[H]$. In addition, every $r \in \Rset^H$ satisfies $0 \leq r(\tau) \leq R$ for any trajectory $\tau$.
\end{assumption}

\begin{assumption}[Transition realizability] \label{assumption:transition function class}
    We have $\pstar_h \in \Pset$ for all $h\in[H]$.
\end{assumption}

Additionally, we introduce the value function class and assume it is bounded. Note that every $\Q{\pi}{h, r}$ satisfies the condition $\norm{f}{\infty}\leq R$ due to Assumption~\ref{assumption:reward function class}.

\begin{assumption}[Value function class] \label{assumption:value function class}
    For any $h\in[H]$, $r\in\Rset^H$, and policy $\pi$, we have $\Q{\pi}{h,r} \in \mathcal{F}$. In addition, every $f\in\mathcal{F}$ satisfies $0 \leq f(s,a) \leq R$ for all $(s,a)\in\Sspace\times\Aspace$.
\end{assumption}

The following assumption defines the trajectory concentrability coefficient between the optimal policy $\pistar$ and the reference policy $\piref$. 

\begin{assumption}[Trajectory concentrability] \label{assumption:concentrability}
    There exists a finite constant $\Ctr$ such that the behavior policy $\piref$ and the optimal policy $\pistar$ satisfy $\sup_{\tau} \frac{d^{\pistar}(\tau)}{d^{\piref}(\tau)} \leq \Ctr$.
\end{assumption}

The bounded $\Ctr$ ensures that the support of $d^{\piref}$ sufficiently covers the support of $d^{\pistar}$, similar to the concentrability condition in \citet{zhan2024provable}\footnote{Our analysis remains valid under an alternative definition based on the reward model error ratio, similar to that used by \citet{zhan2024provable}.}.
As a result, we expect $\dataTraj$ to contain high-quality trajectories.
The lower bound in \citet{zhan2024provable} shows that the trajectory concentrability is essential in offline PbRL. This implies that offline PbRL is strictly more challenging than offline standard RL, where step-wise concentrability is sufficient to achieve a performance guarantee~\citep{ueharapessimistic}.
We now present the sample complexity bound.

\begin{theorem} \label{thm:upper bound}
    Suppose Assumptions~\ref{assumption:reward function class},\ref{assumption:transition function class},
    \ref{assumption:value function class}, and \ref{assumption:concentrability} hold. With probability at least $1-\delta$, Algorithm~\ref{alg:APPO} with $\lambda = \Theta(\Ctr), \lambda>\Ctr, \eta = \sqrt{\frac{2\log|\Aspace|}{R^2 T}}$ achieves
    \begin{align*}
        &\V{\pistar}{1, \rstar} - \V{\hat{\pi}}{1, \rstar} \\
        &\leq \mathcal{O}\left( \Ctr \sqrt{\frac{ \kappa^2 H}{M}\log\frac{|\Rset|}{\delta}} + RH \sqrt{\frac{1}{N}\max\left\{ HT\log\frac{H|\mathcal{F}|}{\delta},  \log\frac{H|\Pset|}{\delta} \right\}} + RH\sqrt{\frac{\log|\Aspace|}{T}} \right) .
    \end{align*}
    Setting $T = \Theta\left(\frac{R^2H^2 \log |\Aspace| }{\epsilon^2} \right)$, $N = \Theta\left(\max\left\{ \frac{R^4 H^5 \log|\Aspace| \log(H|\mathcal{F}|/\delta)}{\epsilon^4}, \frac{R^2 H^2 \log(H|\mathcal{P}|/\delta)}{\epsilon^2} \right\}\right)$, and $ M = \Theta\left( \frac{\Ctr^2 \kappa^2 H \log(|\Rset|/\delta)}{\epsilon^2} \right)$,
    Algorithm~\ref{alg:APPO} achieves $\epsilon$-optimal policy, i.e. $\V{\pistar}{1, \rstar} - \V{\bar{\pi}}{1, \rstar} \leq \epsilon$.
\end{theorem}

\textbf{Discussion on Theorem~\ref{thm:upper bound}.}
Our analysis naturally extends to infinite function classes by applying the standard covering number argument, replacing the cardinalities $|\Rset|$, $|\Pset|$, and $|\mathcal{F}|$ with covering numbers. 
To our knowledge, FREEHAND-transition~\citet{zhan2024provable} is the only statistically efficient algorithm for offline PbRL in stochastic MDPs.
Our sample complexity bound matches theirs for labeled data ($M$). However, FREEHAND-transition requires $\Theta\left(\frac{C_P^2 R^2H^2\log(H|\Pset|/\delta)}{\epsilon^2}\right)$ unlabeled trajectories where $C_P$ is the trajectory concentrability for transition~\footnote{ \citet{zhan2024provable} consider reward functions defined over trajectories, thus their reward class $\mathcal{G}_r$ is comparable to our $\Rset^H$.
They use bracketing numbers in their bound, but we write here $|\Pset|$ for simplicity.}. 
This highlights a trade-off: While FREEHAND-transition has tighter bounds for unlabeled data ($N$), it is computationally intractable.
That is, FREEHAND-transition requires solving a nearly intractable nested optimization problem.
Therefore, our \texttt{APPO} is the first offline PbRL algorithm to achieve both provable statistical efficiency and computational efficiency.

\textbf{Proof Sketch.}
We outline the proof of Theorem~\ref{thm:upper bound}, where the detailed proof is deferred to Appendix~\ref{sec:proof of APPO}. The key observation is our novel sub-optimality decomposition:
\begin{align*}
    &\V{\pistar}{1, \rstar} - \V{\pi^t}{1, \rstar} \\
    &= \underbrace{\V{\pistar}{1,\rstar-\rhat} - \V{\piref}{1,\rstar-\rhat}}_{\text{(I) : MLE error}}
    + \underbrace{ \V{\pistar}{1,\rhat-r^t} - \V{\piref}{1,\rhat-r^t} - \V{\pi^t}{1,\rstar} + \V{\piref}{1, \rstar} + \V{\pi^t}{1,r^t} - \V{\piref}{1, r^t}}_{\text{(II) : Optimization error}}
    + \underbrace{\V{\pistar}{1, r^t} - \V{\pi^t}{1, r^t} }_{\text{(III) : Policy update regret}},
\end{align*}
where $r^t = r^{\pi^t}_{\pstar, f^t}$, and the initial state $s_1$ is omitted here for readability.
The term (I) is bounded by a standard MLE guarantee (Lemma~\ref{lemma:reward MLE error}), while the policy update rule ensures that the sum of terms (III) over $T$ steps is bounded (Lemma~\ref{lemma:NPG regret}). 
For (II), Assumption~\ref{assumption:concentrability} and $\lambda>\Ctr$ imply that
\begin{align*}
    \V{\pistar}{1,\rhat-r^t} - \V{\piref}{1,\rhat-r^t} &= \mathbb{E}_{\tau^0\sim \pistar, \tau^1 \sim \piref} \left[ r^t(\tau^0) - \rhat(\tau^0) - r^t(\tau^1) + \rhat(\tau^1) \right] \\
    &\leq  \Ctr \mathbb{E}_{\tau^0\sim \pistar, \tau^1 \sim \piref} \left[ \left| r^t(\tau^0) - \rhat(\tau^0) - r^t(\tau^1) + \rhat(\tau^1) \right| \right]
    \leq \lambda \mathcal{E}(f^t; \pstar, \rhat).
\end{align*}
Observe that \texttt{APPO} approximately solves the optimization problem in \eqref{eqn:reparameterized optimization} (Lemma~\ref{lemma:optimization error}), and this optimization problem is equivalent to
$\argmin_{f\in\mathcal{F}^H} \{ \V{\pi^t}{1, r^{\pi^t}_{\pstar, f}} - \V{\piref}{1,r^{\pi^t}_{\pstar, f}} + \lambda \mathcal{E}(f; \pstar, \rhat) \}$.
Since $r^{\pi^t}_{\pstar, f^t} = r^t$ and $r^{\pi^t}_{\pstar, \Q{\pi^t}{}} = \rstar$, it follows that
\begin{align*}
    \V{\pi^t}{1, r^t} - \V{\piref}{1,r^t} + \lambda \mathcal{E}(f^t; \pstar, \rhat) \leq \V{\pi^t}{1, \rstar} - \V{\piref}{1,\rstar} + \lambda \mathcal{E}(\Q{\pi^t}{}; \pstar, \rhat) + \epsilon.
\end{align*}
where $\epsilon$ represents an approximation error.
Thus, we obtain (II)$\leq \epsilon$. Combining the results with $\V{\pistar}{1, \rstar} - \V{\hat{\pi}}{1, \rstar} = \frac{1}{T} \sum^T_{t=1} \left( \V{\pistar}{1, \rstar} - \V{\pi^t}{1, \rstar} \right)$, we complete the proof.

\section{Practical Implementation of \texttt{APPO}}

While providing strong statistical guarantees, $\texttt{APPO}$ allows practical implementation with neural networks, leveraging advanced training techniques from deep learning.
In this section, we present a practical version of $\texttt{APPO}$ tailored for deep PbRL.
The pseudo-code is outlined in Algorithm~\ref{alg:practical}. For practical implementation, we assume the standard discounted MDP setting in deep PbRL~\citep{christiano2017deep}, where trajectory segments of length $L$ are given and preference labels are assigned to segment pairs.

\textbf{Reward Learning.}
While our theoretical analysis is based on the maximum likelihood estimator, any reward learning strategy can be employed.
This flexibility allows $\texttt{APPO}$ to benefit from state-of-the-arts preference learning methods, such as data augmentation~\citep{park2022surf} and active query techniques~\citep{shin2023benchmarks,hwang2024sequential,choi2024listwise}.

\textbf{Training Value Functions.}
Given a parameterized policy $\pi_{\theta}$ and an action-value function $Q_{\phi}$, the optimization objective in \eqref{eqn:reparameterized optimization} can be adapted to the discounted setting as follows:
\begin{align*}
    &\argmin_{\phi} \mathbb{E}_{(s,a)\sim d^{\piref}}\left[ (Q_{\phi} \circ \pi_{\theta})(s) - Q_{\phi}(s,a) \right] + \lambda \mathbb{E}_{(\tau_0, \tau_1) \sim \piref}\left[ \left| (r^{\theta}_{\phi} - \rhat)(\tau^0) - (r^{\theta}_{\phi} - \rhat)(\tau^1) \right| \right]
\end{align*}
where $r^{\theta}_{\phi}(\tau) = \sum^L_{l=1} \left( Q_{\phi}(s_l,a_l) - \gamma (Q_{\phi}\circ \pi_{\theta})(s_{l+1}) \right)$ for the segment $\tau = (s_1,a_1,\dots,s_L,a_L)$.
We employ the approximation $\pstar (Q_{\phi}\circ \pi_{\theta})(s_l,a_l) \approx (Q_{\phi}\circ \pi_{\theta})(s_{l+1})$ to avoid the need for a transition model.
Additionally, to stabilize training, we apply the clipped double Q-learning trick~\citep{fujimoto2018addressing,haarnoja2018soft} and maintain a separate value-function $V_{\psi}$.
Given mini-batches of trajectory pairs $\mathcal{B}_{\text{traj}}$ and transition tuples $\mathcal{B}_{\text{tup}}$, each action-value function $Q_{\phi^i}$ is trained by minimizing $\mathcal{L}^{\lambda}_{\phi^i} = \lambda \mathcal{L}^{\text{adv}}_{\phi^i} + \mathcal{E}_{\phi^i}$ (where $\lambda$ is moved to the first term, without loss of generality), defined as follows:
\begin{align} 
    &\mathcal{L}^{\text{adv}}_{\phi^i}(\mathcal{B}_{\text{tup}}) = \E_{(s,a)\sim\mathcal{B}_{\text{tup}}}\left[ Q_{\phi^i}(s, \pi_{\theta}(s)) - Q_{\phi^i}(s,a) \right], \notag\\ 
    \text{and }&\mathcal{E}_{\phi^i}(\mathcal{B}_{\text{traj}}) = \E_{(\tau^0,\tau^1)\sim\mathcal{B}_{\text{traj}}}\left[\left| \{r^{\psi}_{\phi^i}(\tau^0)-r^{\psi}_{\phi^i}(\tau^1)\} - \{\rhat(\tau^0) - \rhat(\tau^1)\}  \right|\right]. \label{eqn:practical q loss}
\end{align} 
Here, we use the notation $r^{\psi}_{\phi^i}(\tau) = \sum^L_{l=1} \left( Q_{\phi^i}(s_l,a_l) - \gamma V_{\psi}(s_{h+l})  \right)$, and $\pi_{\theta}(s)$ denotes an action sampled from $\pi_{\theta}(\cdot\mid s)$.
Given target Q-networks $\{\bar{\phi}^i\}_{i\in\{1,2\}}$, we train $V_{\psi}$ by minimizing 
\begin{align} \label{eqn:practical v loss}
    \mathcal{L}_{\psi}(\mathcal{B}_{\text{tup}}) = \E_{s\sim\mathcal{B}_{\text{tup}}}\left[ \left( V_{\psi}(s) - \min_{i\in\{1,2\}} Q_{\bar{\phi}^i}(s_{h+1}, \pi_{\theta}(s_{h+1})) \right)^2 \right],
\end{align}
Intuitively, the term $\mathcal{L}^{\text{adv}}_{\phi}$ ensures conservatism by regularizing $Q_{\phi}$ to have lower values near $d^{\pi_{\theta}}$ and higher values near $d^{\piref}$.
Additional insight can be gained by rearranging the integrand of $\mathcal{E}_{\phi}$:
\begin{align*}
    r^{\psi}_{\phi^i}(\tau) - \rhat(\tau) = \sum^L_{l=1} \left( Q_{\phi^i}(s_l,a_l) - \rhat(s_l,a_l) - \gamma V_{\psi}(s_{l+1})) \right).
\end{align*}
This expression represents the sum of the TD errors evaluated over the segment $\tau$. Thus, the loss $\mathcal{E}_{\phi}$ minimizes the difference in trajectory TD errors between $\tau^0$ and $\tau^1$.

\textbf{Training Policy.}
We directly optimize the policy using the loss function in \eqref{eqn:practical actor loss}.
The entropy regularization term is similar to that in SAC~\citep{haarnoja2018soft}, except that we use a randomly sampled $Q_{\phi^i}$ instead of the clipped value $\min_{i\in[1,2]} Q_{\phi^i}$. The policy loss is given by:
\begin{align} \label{eqn:practical actor loss}
    \mathcal{L}_{\theta}(\mathcal{B}_{\text{tup}}) = \mathbb{E}_{s\sim\mathcal{B}_{\text{tup}}}\left[  Q_{\phi^i}(s, \pi_{\theta}(s)) - \alpha \pi_{\theta}(s, \pi_{\theta}(s)) \right], \,\, i\sim \text{Unif}\{1,2\}
\end{align}

\begin{table}[!t]
\centering
\setlength{\tabcolsep}{5pt} 
{\footnotesize
\begin{tabular}{l|r|rrrrr}
\toprule
\makecell[l]{Dataset \&\\\# of feedback} & Oracle & MR & PT & DPPO & IPL & 
\texttt{APPO}$\,$(ours)
\\ \midrule\midrule

BPT-500         & 88.33\,{\tiny$\pm4.76$}  & 10.08\,{\tiny$\pm7.57$}  & 22.87\,{\tiny$\pm9.06$}  & 3.93\,{\tiny$\pm4.34$}   & 34.73\,{\tiny$\pm13.9$}  & \textbf{53.52}\,{\tiny$\pm13.9$} \\
box-close-500   & 93.40\,{\tiny$\pm3.10$}  & \textbf{29.12}\,{\tiny$\pm13.2$}  & 0.33\,{\tiny$\pm1.16$}   & 10.20\,{\tiny$\pm11.5$}  & 5.93\,{\tiny$\pm5.81$}   & \textbf{18.24}\,{\tiny$\pm15.6$} \\
dial-turn-500   & 75.40\,{\tiny$\pm5.47$}  & 61.44\,{\tiny$\pm6.08$}  & \textbf{68.67}\,{\tiny$\pm12.4$}  & 26.67\,{\tiny$\pm22.2$}  & 31.53\,{\tiny$\pm12.5$}  & \textbf{80.96}\,{\tiny$\pm4.49$} \\
sweep-500       & 98.33\,{\tiny$\pm1.87$}  & \textbf{86.96}\,{\tiny$\pm6.93$}  & 43.07\,{\tiny$\pm24.6$}  & 10.47\,{\tiny$\pm15.8$}  & 27.20\,{\tiny$\pm23.8$}  & 26.80\,{\tiny$\pm5.32$} \\
BPT-wall-500    & 56.27\,{\tiny$\pm6.32$}  & 0.32\,{\tiny$\pm0.30$}  & 0.87\,{\tiny$\pm1.43$}   & 0.80\,{\tiny$\pm1.51$}   & 8.93\,{\tiny$\pm9.84$}   & \textbf{64.32}\,{\tiny$\pm21.0$} \\
sweep-into-500  & 78.80\,{\tiny$\pm7.96$}  & \textbf{28.40}\,{\tiny$\pm5.47$}  & 20.53\,{\tiny$\pm8.26$}  & 23.07\,{\tiny$\pm7.02$}  & \textbf{32.20}\,{\tiny$\pm7.35$}  & 24.08\,{\tiny$\pm5.91$} \\
drawer-open-500 & 100.00\,{\tiny$\pm0.00$} & \textbf{98.00}\,{\tiny$\pm2.32$}  & 88.73\,{\tiny$\pm11.6$}  & 35.93\,{\tiny$\pm11.2$}  & 19.00\,{\tiny$\pm13.6$}  & 87.68\,{\tiny$\pm10.0$} \\
lever-pull-500  & 98.47\,{\tiny$\pm1.77$}  & \textbf{79.28}\,{\tiny$\pm2.95$}  & \textbf{82.40}\,{\tiny$\pm22.7$}  & 10.13\,{\tiny$\pm12.2$}  & 31.20\,{\tiny$\pm15.8$}  & \textbf{75.76}\,{\tiny$\pm7.17$} \\ 
\midrule

BPT-1000         & 88.33\,{\tiny$\pm4.76$} & 8.48\,{\tiny$\pm5.80$}   & 18.27\,{\tiny$\pm10.6$}  & 3.20\,{\tiny$\pm3.04$}   & 36.67\,{\tiny$\pm17.4$}  & \textbf{59.04}\,{\tiny$\pm19.0$} \\
box-close-1000   & 93.40\,{\tiny$\pm3.10$} & \textbf{27.04}\,{\tiny$\pm14.5$} & 2.27\,{\tiny$\pm2.86$}   & 9.33\,{\tiny$\pm9.60$}   & 6.73\,{\tiny$\pm8.41$}   & \textbf{34.24}\,{\tiny$\pm18.5$} \\
dial-turn-1000   & 75.40\,{\tiny$\pm5.47$} & 69.44\,{\tiny$\pm4.70$}  & 68.80\,{\tiny$\pm5.50$}  & 36.40\,{\tiny$\pm21.9$}  & 43.93\,{\tiny$\pm13.4$}  & \textbf{81.44}\,{\tiny$\pm6.73$} \\
sweep-1000       & 98.33\,{\tiny$\pm1.87$} & \textbf{87.52}\,{\tiny$\pm7.87$} & 29.13\,{\tiny$\pm14.6$}  & 8.73\,{\tiny$\pm16.4$}   & 38.33\,{\tiny$\pm24.9$}  & 17.36\,{\tiny$\pm12.4$} \\
BPT-wall-1000    & 56.27\,{\tiny$\pm6.32$} & 0.48\,{\tiny$\pm0.47$}  & 2.13\,{\tiny$\pm2.96$}   & 0.27\,{\tiny$\pm0.85$}   & 14.07\,{\tiny$\pm11.5$}  & \textbf{62.96}\,{\tiny$\pm18.4$} \\
sweep-into-1000  & 78.80\,{\tiny$\pm7.96$} & \textbf{26.00}\,{\tiny$\pm5.53$} & 20.27\,{\tiny$\pm7.84$}  & \textbf{23.33}\,{\tiny$\pm7.80$}  & \textbf{30.40}\,{\tiny$\pm7.74$}  & 18.16\,{\tiny$\pm11.1$} \\
drawer-open-1000 & 100.00\,{\tiny$\pm0.00$} & \textbf{98.40}\,{\tiny$\pm2.82$} & \textbf{95.40}\,{\tiny$\pm7.27$}  & 36.47\,{\tiny$\pm7.30$}   & 28.53\,{\tiny$\pm18.4$}  & \textbf{98.56}\,{\tiny$\pm2.68$} \\
lever-pull-1000  & 98.47\,{\tiny$\pm1.77$} & \textbf{88.96}\,{\tiny$\pm3.94$} & 72.93\,{\tiny$\pm10.2$}  & 8.53\,{\tiny$\pm9.96$}   & 40.40\,{\tiny$\pm17.4$}  & 76.96\,{\tiny$\pm4.40$} \\
\midrule
Average Rank & - & 2.316 & 3.125 & 4.375 & 3.063 & \textbf{2.125} \\
\bottomrule
\end{tabular}
} 
\caption{Success rates on Meta-World \texttt{medium-replay} dataset with $500$ and $1000$ preference feedback samples, averaged over $5$ random seeds. The results of baselines, Oracle, PT, DPPO, and IPL, are taken from \citet{choi2024listwise}, where Oracle refers to the policy trained using IQL with ground-truth rewards. The abbreviation BPT stands for button-press-topdown.}
\label{tab:main experiment}
\end{table}

\section{Experiments} \label{sec:experiments}

\textbf{Datasets and Evaluation.}
We evaluate our proposed algorithm on the Meta-World~\citep{yu2020meta} $\texttt{medium-replay}$ and \texttt{medium-expert} datasets from \citet{choi2024listwise}.
Our main experiments use the \texttt{medium-replay} dataset, while the experiments with the \texttt{medium-expert} dataset are presented in Appendix~\ref{sec:additional experiments}.
These datasets have a favorable property: they are not learnable with incorrect rewards (random or constant).
This property is crucial for evaluating offline RL algorithms since their survival instinct can allow them to perform well even with completely incorrect reward signals~\citep{li2024survival}.
For more details on the dataset, see \citet{choi2024listwise}.
Following the experiment protocol of \citet{choi2024listwise}, the preference dataset consists of pairs of randomly sampled trajectory segments of length $25$.
The preference label is generated based on the ground truth reward, where a $(0,1)$ label is assigned if the trajectory rewards differ by more than a threshold of 12.5, and a $(0.5,0.5)$ label is assigned otherwise.
We evaluate algorithm performance using the success rate for each task, which indicates whether the agent successfully completes the task.

\textbf{Algorithms.}
We consider four offline PbRL algorithms as baselines: Markovian Reward (MR), Preference Transformer (PT)~\citep{kim2023preference}, Direct Preference-based Policy Optimization (DPPO)~\citep{an2023direct}, and Inverse Preference Learning (IPL)~\citep{hejna2024inverse}.
MR is an instance of IQL~\citep{kostrikov2022offline} trained with a Markovian reward model, while PT assumes a general sequential reward model implemented using a Transformer~\citep{vaswani2017attention} architecture.
DPPO directly optimizes the policy without using a reward model, while the other baseline methods are based on IQL~\citep{kostrikov2022offline}.
We evaluate the practical version of \texttt{APPO} in Algorithm~\ref{alg:practical}, using the same reward model as MR and setting $\lambda = 0.03$. Further details are in Appendix~\ref{sec:experimental details}.

\begin{figure}[!t]
    \centering
    \includegraphics[width=\linewidth]{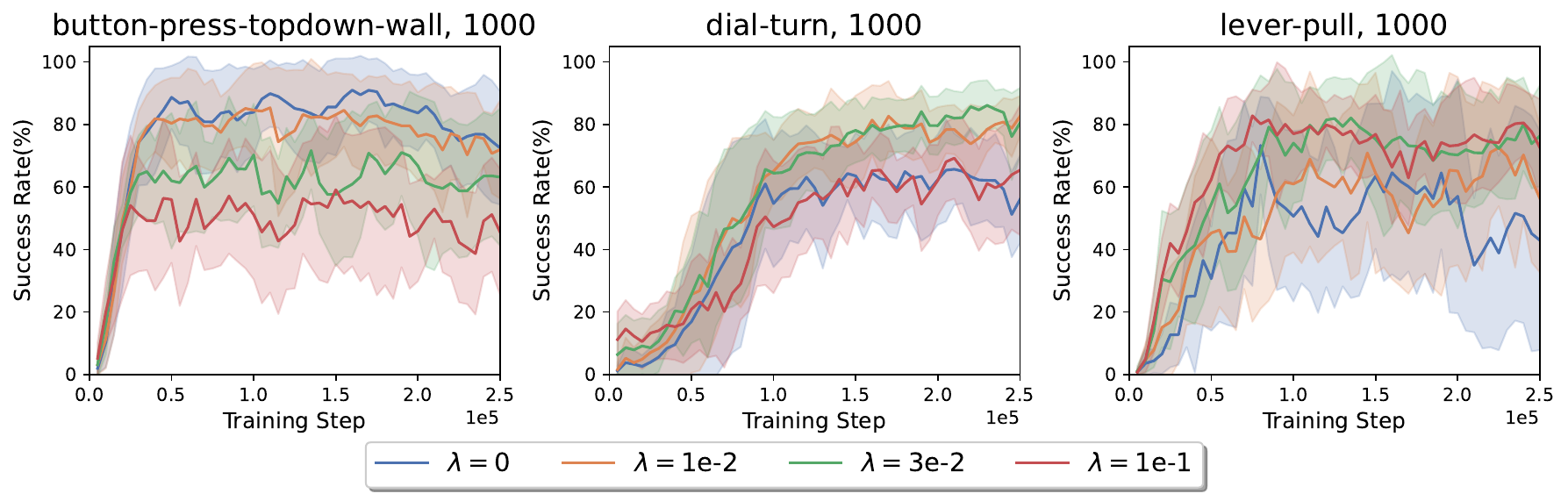}
    \caption{Effect of the conservatism regularizer $\lambda$.}
    \label{fig:hparameter}
\end{figure}

\begin{figure}[t]
    \centering
    \includegraphics[width=\linewidth]{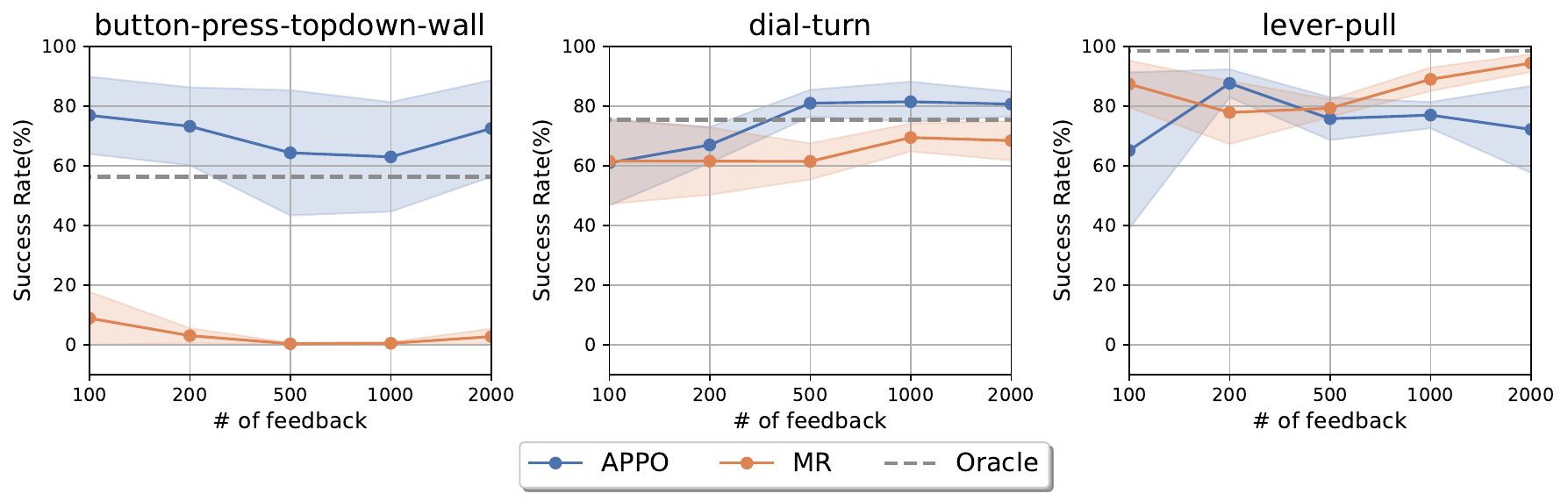}
    \caption{Success rates of \texttt{APPO} and MR, with varying number of preference feedback samples.}
    \label{fig:data size}
\end{figure}

\subsection{Evaluation Results}
Table~\ref{tab:main experiment} shows the performance of algorithms on the Meta-World control tasks. \texttt{APPO} outperforms or shows comparable performance in almost every dataset.
Notably, \texttt{APPO} outperforms the policy trained with ground truth rewards in the dial-turn and button-press-topdown-wall datasets.
We also observe that MR is a strong baseline, as reported in previous works~\citep{hejna2024inverse,choi2024listwise}.
These results suggest that \texttt{APPO} performs comparably to state-of-the-art baselines, even with a provable statistical guarantee.

\textbf{Effect of Conservatism Regularizer.}
We investigate the effect of conservatism regularizer $\lambda$, the coefficient to balance the adversarial loss $\mathcal{L}^{\text{adv}}_{\phi}$ and the trajectory-pair $\ell_1$ loss $\mathcal{E}_{\phi}$.
In Figure~\ref{fig:hparameter}, \texttt{APPO} learns successfully with a wide range of $\lambda$, but a properly tuned $\lambda$ improves performance and stability.
We note that \texttt{APPO} has only one algorithmic hyperparameter $\lambda$, whereas IQL-based algorithms (MR, PT, IPL) require at least two (expectile parameter and temperature).
DPPO, on the other hand, has two hyperparameters (conservatism regularizer and smoothness regularizer).

\textbf{Effect of Preference Dataset Size.}
In PbRL, learning from small preference datasets is desired for cost-efficient learning. We evaluate the effect of preference dataset size on \texttt{APPO}'s performance, varying the number of feedback samples from $100$ to $2000$.
Figure~\ref{fig:data size} shows that \texttt{APPO} is robust to preference dataset size, displaying variance comparable to MR, a strong baseline, as shown in Table~\ref{tab:main experiment}. Note that \texttt{APPO} outperforms a policy trained with ground truth rewards, using only $100$ preference feedback samples.

\section*{Reproducibility}
We describe the experimental details in Section~\ref{sec:experiments} and Section~\ref{sec:experimental details}, including training protocol and neural network architecture. Our code is available at \url{https://github.com/oh-lab/APPO.git}. As explained in Section~\ref{sec:experiments}, we use the Meta-World \texttt{medium-replay} dataset from \citet{choi2024listwise}, which is available in their official repository along with download instructions.

\section*{Acknowledgements}
This work was supported by the National Research Foundation of Korea~(NRF) grant funded by the Korea government~(MSIT) (No.  RS-2022-NR071853 and RS-2023-00222663) and by AI-Bio Research Grant through Seoul National University.

\bibliography{ref}

\begin{thebibliography}{68}
\providecommand{\natexlab}[1]{#1}
\providecommand{\url}[1]{\texttt{#1}}
\expandafter\ifx\csname urlstyle\endcsname\relax
  \providecommand{\doi}[1]{doi: #1}\else
  \providecommand{\doi}{doi: \begingroup \urlstyle{rm}\Url}\fi

\bibitem[An et~al.(2023)An, Lee, Zuo, Kosaka, Kim, and Song]{an2023direct}
Gaon An, Junhyeok Lee, Xingdong Zuo, Norio Kosaka, Kyung-Min Kim, and Hyun~Oh Song.
\newblock Direct preference-based policy optimization without reward modeling.
\newblock \emph{Advances in Neural Information Processing Systems}, 36:\penalty0 70247--70266, 2023.

\bibitem[Artzner(1997)]{artzner1997thinking}
Philippe Artzner.
\newblock Thinking coherently.
\newblock \emph{Risk}, 10:\penalty0 68--71, 1997.

\bibitem[Bai et~al.(2022)Bai, Jones, Ndousse, Askell, Chen, DasSarma, Drain, Fort, Ganguli, Henighan, et~al.]{bai2022training}
Yuntao Bai, Andy Jones, Kamal Ndousse, Amanda Askell, Anna Chen, Nova DasSarma, Dawn Drain, Stanislav Fort, Deep Ganguli, Tom Henighan, et~al.
\newblock Training a helpful and harmless assistant with reinforcement learning from human feedback.
\newblock \emph{arXiv preprint arXiv:2204.05862}, 2022.

\bibitem[Bhardwaj et~al.(2024)Bhardwaj, Xie, Boots, Jiang, and Cheng]{bhardwaj2024adversarial}
Mohak Bhardwaj, Tengyang Xie, Byron Boots, Nan Jiang, and Ching-An Cheng.
\newblock Adversarial model for offline reinforcement learning.
\newblock \emph{Advances in Neural Information Processing Systems}, 36, 2024.

\bibitem[Bradley \& Terry(1952)Bradley and Terry]{bradley1952rank}
Ralph~Allan Bradley and Milton~E Terry.
\newblock Rank analysis of incomplete block designs: I. the method of paired comparisons.
\newblock \emph{Biometrika}, 39\penalty0 (3/4):\penalty0 324--345, 1952.

\bibitem[Brohan et~al.(2022)Brohan, Brown, Carbajal, Chebotar, Dabis, Finn, Gopalakrishnan, Hausman, Herzog, Hsu, et~al.]{brohan2022rt}
Anthony Brohan, Noah Brown, Justice Carbajal, Yevgen Chebotar, Joseph Dabis, Chelsea Finn, Keerthana Gopalakrishnan, Karol Hausman, Alex Herzog, Jasmine Hsu, et~al.
\newblock Rt-1: Robotics transformer for real-world control at scale.
\newblock \emph{arXiv preprint arXiv:2212.06817}, 2022.

\bibitem[Brown et~al.(2019)Brown, Goo, Nagarajan, and Niekum]{brown2019extrapolating}
Daniel Brown, Wonjoon Goo, Prabhat Nagarajan, and Scott Niekum.
\newblock Extrapolating beyond suboptimal demonstrations via inverse reinforcement learning from observations.
\newblock In \emph{International conference on machine learning}, pp.\  783--792. PMLR, 2019.

\bibitem[Cai et~al.(2020)Cai, Yang, Jin, and Wang]{cai2020provably}
Qi~Cai, Zhuoran Yang, Chi Jin, and Zhaoran Wang.
\newblock Provably efficient exploration in policy optimization.
\newblock In \emph{International Conference on Machine Learning}, pp.\  1283--1294. PMLR, 2020.

\bibitem[Chang et~al.(2024)Chang, Shan, Oertell, Brantley, Misra, Lee, and Sun]{chang2024dataset}
Jonathan~D Chang, Wenhao Shan, Owen Oertell, Kiant{\'e} Brantley, Dipendra Misra, Jason~D Lee, and Wen Sun.
\newblock Dataset reset policy optimization for rlhf.
\newblock \emph{arXiv preprint arXiv:2404.08495}, 2024.

\bibitem[Chen et~al.(2022)Chen, Zhong, Yang, Wang, and Wang]{chen2022human}
Xiaoyu Chen, Han Zhong, Zhuoran Yang, Zhaoran Wang, and Liwei Wang.
\newblock Human-in-the-loop: Provably efficient preference-based reinforcement learning with general function approximation.
\newblock In \emph{International Conference on Machine Learning}, pp.\  3773--3793. PMLR, 2022.

\bibitem[Chen et~al.(2023)Chen, Du, Hu, Wang, Wu, and Huang]{chen2023provably}
Yu~Chen, Yihan Du, Pihe Hu, Siwei Wang, Desheng Wu, and Longbo Huang.
\newblock Provably efficient iterated cvar reinforcement learning with function approximation and human feedback.
\newblock In \emph{The Twelfth International Conference on Learning Representations}, 2023.

\bibitem[Cheng et~al.(2022)Cheng, Xie, Jiang, and Agarwal]{cheng2022adversarially}
Ching-An Cheng, Tengyang Xie, Nan Jiang, and Alekh Agarwal.
\newblock Adversarially trained actor critic for offline reinforcement learning.
\newblock In \emph{International Conference on Machine Learning}, pp.\  3852--3878. PMLR, 2022.

\bibitem[Cheng et~al.(2024)Cheng, Yang, Li, Dai, Hu, Cao, Du, and Li]{cheng2024adversarial}
Pengyu Cheng, Yifan Yang, Jian Li, Yong Dai, Tianhao Hu, Peixin Cao, Nan Du, and Xiaolong Li.
\newblock Adversarial preference optimization: Enhancing your alignment via rm-llm game.
\newblock In \emph{Findings of the Association for Computational Linguistics ACL 2024}, pp.\  3705--3716, 2024.

\bibitem[Choi et~al.(2024)Choi, Jung, Ahn, and Moon]{choi2024listwise}
Heewoong Choi, Sangwon Jung, Hongjoon Ahn, and Taesup Moon.
\newblock Listwise reward estimation for offline preference-based reinforcement learning.
\newblock In \emph{Forty-first International Conference on Machine Learning}, 2024.

\bibitem[Christiano et~al.(2017)Christiano, Leike, Brown, Martic, Legg, and Amodei]{christiano2017deep}
Paul~F Christiano, Jan Leike, Tom Brown, Miljan Martic, Shane Legg, and Dario Amodei.
\newblock Deep reinforcement learning from human preferences.
\newblock \emph{Advances in neural information processing systems}, 30, 2017.

\bibitem[Du et~al.(2024)Du, Winnicki, Dalal, Mannor, and Srikant]{du2024explorationdriven}
Yihan Du, Anna Winnicki, Gal Dalal, Shie Mannor, and R.~Srikant.
\newblock Exploration-driven policy optimization in {RLHF}: Theoretical insights on efficient data utilization.
\newblock In \emph{Forty-first International Conference on Machine Learning}, 2024.

\bibitem[Fujimoto et~al.(2018)Fujimoto, Hoof, and Meger]{fujimoto2018addressing}
Scott Fujimoto, Herke Hoof, and David Meger.
\newblock Addressing function approximation error in actor-critic methods.
\newblock In \emph{International conference on machine learning}, pp.\  1587--1596. PMLR, 2018.

\bibitem[Haarnoja et~al.(2018)Haarnoja, Zhou, Abbeel, and Levine]{haarnoja2018soft}
Tuomas Haarnoja, Aurick Zhou, Pieter Abbeel, and Sergey Levine.
\newblock Soft actor-critic: Off-policy maximum entropy deep reinforcement learning with a stochastic actor.
\newblock In \emph{International conference on machine learning}, pp.\  1861--1870. PMLR, 2018.

\bibitem[Hejna \& Sadigh(2024)Hejna and Sadigh]{hejna2024inverse}
Joey Hejna and Dorsa Sadigh.
\newblock Inverse preference learning: Preference-based rl without a reward function.
\newblock \emph{Advances in Neural Information Processing Systems}, 36, 2024.

\bibitem[Hejna et~al.(2024)Hejna, Rafailov, Sikchi, Finn, Niekum, Knox, and Sadigh]{hejna2024contrastive}
Joey Hejna, Rafael Rafailov, Harshit Sikchi, Chelsea Finn, Scott Niekum, W.~Bradley Knox, and Dorsa Sadigh.
\newblock Contrastive preference learning: Learning from human feedback without reinforcement learning.
\newblock In \emph{The Twelfth International Conference on Learning Representations}, 2024.

\bibitem[Hejna~III \& Sadigh(2023)Hejna~III and Sadigh]{hejna2023few}
Donald~Joseph Hejna~III and Dorsa Sadigh.
\newblock Few-shot preference learning for human-in-the-loop rl.
\newblock In \emph{Conference on Robot Learning}, pp.\  2014--2025. PMLR, 2023.

\bibitem[Hwang et~al.(2024)Hwang, Lee, Kee, Kim, Lee, and Oh]{hwang2024sequential}
Minyoung Hwang, Gunmin Lee, Hogun Kee, Chan~Woo Kim, Kyungjae Lee, and Songhwai Oh.
\newblock Sequential preference ranking for efficient reinforcement learning from human feedback.
\newblock \emph{Advances in Neural Information Processing Systems}, 36, 2024.

\bibitem[Ibarz et~al.(2018)Ibarz, Leike, Pohlen, Irving, Legg, and Amodei]{ibarz2018reward}
Borja Ibarz, Jan Leike, Tobias Pohlen, Geoffrey Irving, Shane Legg, and Dario Amodei.
\newblock Reward learning from human preferences and demonstrations in atari.
\newblock \emph{Advances in neural information processing systems}, 31, 2018.

\bibitem[Kakade \& Langford(2002)Kakade and Langford]{kakade2002approximately}
Sham Kakade and John Langford.
\newblock Approximately optimal approximate reinforcement learning.
\newblock In \emph{Proceedings of the Nineteenth International Conference on Machine Learning}, pp.\  267--274, 2002.

\bibitem[Kalashnikov et~al.(2018)Kalashnikov, Irpan, Pastor, Ibarz, Herzog, Jang, Quillen, Holly, Kalakrishnan, Vanhoucke, et~al.]{kalashnikov2018scalable}
Dmitry Kalashnikov, Alex Irpan, Peter Pastor, Julian Ibarz, Alexander Herzog, Eric Jang, Deirdre Quillen, Ethan Holly, Mrinal Kalakrishnan, Vincent Vanhoucke, et~al.
\newblock Scalable deep reinforcement learning for vision-based robotic manipulation.
\newblock In \emph{Conference on robot learning}, pp.\  651--673. PMLR, 2018.

\bibitem[Kang et~al.(2023)Kang, Shi, Liu, He, and Wang]{kang2023beyond}
Yachen Kang, Diyuan Shi, Jinxin Liu, Li~He, and Donglin Wang.
\newblock Beyond reward: Offline preference-guided policy optimization.
\newblock In \emph{International Conference on Machine Learning}, pp.\  15753--15768. PMLR, 2023.

\bibitem[Kim et~al.(2023)Kim, Park, Shin, Lee, Abbeel, and Lee]{kim2023preference}
Changyeon Kim, Jongjin Park, Jinwoo Shin, Honglak Lee, Pieter Abbeel, and Kimin Lee.
\newblock Preference transformer: Modeling human preferences using transformers for {RL}.
\newblock In \emph{The Eleventh International Conference on Learning Representations}, 2023.

\bibitem[Kingma \& Ba(2015)Kingma and Ba]{kingma2015adam}
Diederik~P. Kingma and Jimmy Ba.
\newblock Adam: A method for stochastic optimization.
\newblock In \emph{3rd International Conference on Learning Representations (ICLR)}, 2015.
\newblock URL \url{https://arxiv.org/abs/1412.6980}.
\newblock arXiv preprint arXiv:1412.6980.

\bibitem[Kostrikov et~al.(2022)Kostrikov, Nair, and Levine]{kostrikov2022offline}
Ilya Kostrikov, Ashvin Nair, and Sergey Levine.
\newblock Offline reinforcement learning with implicit q-learning.
\newblock In \emph{International Conference on Learning Representations}, 2022.

\bibitem[Lee et~al.(2021)Lee, Smith, and Abbeel]{lee2021pebble}
Kimin Lee, Laura~M Smith, and Pieter Abbeel.
\newblock Pebble: Feedback-efficient interactive reinforcement learning via relabeling experience and unsupervised pre-training.
\newblock In \emph{International Conference on Machine Learning}, pp.\  6152--6163. PMLR, 2021.

\bibitem[Li et~al.(2024)Li, Misra, Kolobov, and Cheng]{li2024survival}
Anqi Li, Dipendra Misra, Andrey Kolobov, and Ching-An Cheng.
\newblock Survival instinct in offline reinforcement learning.
\newblock \emph{Advances in neural information processing systems}, 36, 2024.

\bibitem[Liang et~al.(2022)Liang, Shu, Lee, and Abbeel]{liang2022reward}
Xinran Liang, Katherine Shu, Kimin Lee, and Pieter Abbeel.
\newblock Reward uncertainty for exploration in preference-based reinforcement learning.
\newblock In \emph{International Conference on Learning Representations}, 2022.
\newblock URL \url{https://openreview.net/forum?id=OWZVD-l-ZrC}.

\bibitem[Liu et~al.(2022)Liu, Bai, Du, and Yang]{liu2022meta}
Runze Liu, Fengshuo Bai, Yali Du, and Yaodong Yang.
\newblock Meta-reward-net: Implicitly differentiable reward learning for preference-based reinforcement learning.
\newblock \emph{Advances in Neural Information Processing Systems}, 35:\penalty0 22270--22284, 2022.

\bibitem[MacGlashan et~al.(2017)MacGlashan, Ho, Loftin, Peng, Wang, Roberts, Taylor, and Littman]{macglashan2017interactive}
James MacGlashan, Mark~K Ho, Robert Loftin, Bei Peng, Guan Wang, David~L Roberts, Matthew~E Taylor, and Michael~L Littman.
\newblock Interactive learning from policy-dependent human feedback.
\newblock In \emph{International conference on machine learning}, pp.\  2285--2294. PMLR, 2017.

\bibitem[Makar-Limanov et~al.(2024)Makar-Limanov, Prakash, Goktas, Greenwald, and Ayanian]{makar-limanov2024starlhf}
Jacob Makar-Limanov, Arjun Prakash, Denizalp Goktas, Amy Greenwald, and Nora Ayanian.
\newblock {STA}-{RLHF}: Stackelberg aligned reinforcement learning with human feedback.
\newblock In \emph{Coordination and Cooperation for Multi-Agent Reinforcement Learning Methods Workshop}, 2024.

\bibitem[Mnih(2013)]{mnih2013playing}
Volodymyr Mnih.
\newblock Playing atari with deep reinforcement learning.
\newblock \emph{arXiv preprint arXiv:1312.5602}, 2013.

\bibitem[Nakano et~al.(2021)Nakano, Hilton, Balaji, Wu, Ouyang, Kim, Hesse, Jain, Kosaraju, Saunders, et~al.]{nakano2021webgpt}
Reiichiro Nakano, Jacob Hilton, Suchir Balaji, Jeff Wu, Long Ouyang, Christina Kim, Christopher Hesse, Shantanu Jain, Vineet Kosaraju, William Saunders, et~al.
\newblock Webgpt: Browser-assisted question-answering with human feedback.
\newblock \emph{arXiv preprint arXiv:2112.09332}, 2021.

\bibitem[Neu et~al.(2017)Neu, Jonsson, and G{\'o}mez]{neu2017unified}
Gergely Neu, Anders Jonsson, and Vicen{\c{c}} G{\'o}mez.
\newblock A unified view of entropy-regularized markov decision processes.
\newblock \emph{arXiv preprint arXiv:1705.07798}, 2017.

\bibitem[Novoseller et~al.(2020)Novoseller, Wei, Sui, Yue, and Burdick]{novoseller2020dueling}
Ellen Novoseller, Yibing Wei, Yanan Sui, Yisong Yue, and Joel Burdick.
\newblock Dueling posterior sampling for preference-based reinforcement learning.
\newblock In \emph{Conference on Uncertainty in Artificial Intelligence}, pp.\  1029--1038. PMLR, 2020.

\bibitem[Ouyang et~al.(2022)Ouyang, Wu, Jiang, Almeida, Wainwright, Mishkin, Zhang, Agarwal, Slama, Ray, et~al.]{ouyang2022training}
Long Ouyang, Jeffrey Wu, Xu~Jiang, Diogo Almeida, Carroll Wainwright, Pamela Mishkin, Chong Zhang, Sandhini Agarwal, Katarina Slama, Alex Ray, et~al.
\newblock Training language models to follow instructions with human feedback.
\newblock \emph{Advances in neural information processing systems}, 35:\penalty0 27730--27744, 2022.

\bibitem[Pace et~al.(2024)Pace, Sch{\"o}lkopf, Ratsch, and Ramponi]{pace2024preference}
Aliz{\'e}e Pace, Bernhard Sch{\"o}lkopf, Gunnar Ratsch, and Giorgia Ramponi.
\newblock Preference elicitation for offline reinforcement learning.
\newblock In \emph{ICML 2024 Workshop: Foundations of Reinforcement Learning and Control -- Connections and Perspectives}, 2024.
\newblock URL \url{https://openreview.net/forum?id=YGaRv4UCRh}.

\bibitem[Park et~al.(2022)Park, Seo, Shin, Lee, Abbeel, and Lee]{park2022surf}
Jongjin Park, Younggyo Seo, Jinwoo Shin, Honglak Lee, Pieter Abbeel, and Kimin Lee.
\newblock {SURF}: Semi-supervised reward learning with data augmentation for feedback-efficient preference-based reinforcement learning.
\newblock In \emph{International Conference on Learning Representations}, 2022.
\newblock URL \url{https://openreview.net/forum?id=TfhfZLQ2EJO}.

\bibitem[P{\'a}sztor et~al.(2024)P{\'a}sztor, Kassraie, and Krause]{pasztor2024bandits}
Barna P{\'a}sztor, Parnian Kassraie, and Andreas Krause.
\newblock Bandits with preference feedback: A stackelberg game perspective.
\newblock In \emph{The Thirty-eighth Annual Conference on Neural Information Processing Systems}, 2024.
\newblock URL \url{https://openreview.net/forum?id=wIE991zhXH}.

\bibitem[Rajeswaran et~al.(2020)Rajeswaran, Mordatch, and Kumar]{rajeswaran2020game}
Aravind Rajeswaran, Igor Mordatch, and Vikash Kumar.
\newblock A game theoretic framework for model based reinforcement learning.
\newblock In \emph{International conference on machine learning}, pp.\  7953--7963. PMLR, 2020.

\bibitem[Rigter et~al.(2022)Rigter, Lacerda, and Hawes]{rigter2022rambo}
Marc Rigter, Bruno Lacerda, and Nick Hawes.
\newblock Rambo-rl: Robust adversarial model-based offline reinforcement learning.
\newblock \emph{Advances in neural information processing systems}, 35:\penalty0 16082--16097, 2022.

\bibitem[Russo \& Van~Roy(2013)Russo and Van~Roy]{russo2013eluder}
Daniel Russo and Benjamin Van~Roy.
\newblock Eluder dimension and the sample complexity of optimistic exploration.
\newblock \emph{Advances in Neural Information Processing Systems}, 26, 2013.

\bibitem[Saha et~al.(2023)Saha, Pacchiano, and Lee]{saha2023dueling}
Aadirupa Saha, Aldo Pacchiano, and Jonathan Lee.
\newblock Dueling rl: Reinforcement learning with trajectory preferences.
\newblock In \emph{International Conference on Artificial Intelligence and Statistics}, pp.\  6263--6289. PMLR, 2023.

\bibitem[Schulman et~al.(2015)Schulman, Levine, Abbeel, Jordan, and Moritz]{schulman2015trust}
John Schulman, Sergey Levine, Pieter Abbeel, Michael Jordan, and Philipp Moritz.
\newblock Trust region policy optimization.
\newblock In \emph{International conference on machine learning}, pp.\  1889--1897. PMLR, 2015.

\bibitem[Schulman et~al.(2017)Schulman, Wolski, Dhariwal, Radford, and Klimov]{schulman2017proximal}
John Schulman, Filip Wolski, Prafulla Dhariwal, Alec Radford, and Oleg Klimov.
\newblock Proximal policy optimization algorithms.
\newblock \emph{arXiv preprint arXiv:1707.06347}, 2017.

\bibitem[Shen et~al.(2024)Shen, Yang, and Chen]{shen2024principled}
Han Shen, Zhuoran Yang, and Tianyi Chen.
\newblock Principled penalty-based methods for bilevel reinforcement learning and {RLHF}.
\newblock In \emph{Forty-first International Conference on Machine Learning}, 2024.
\newblock URL \url{https://openreview.net/forum?id=Xb3IXEBYuw}.

\bibitem[Shin et~al.(2023)Shin, Dragan, and Brown]{shin2023benchmarks}
Daniel Shin, Anca Dragan, and Daniel~S. Brown.
\newblock Benchmarks and algorithms for offline preference-based reward learning.
\newblock \emph{Transactions on Machine Learning Research}, 2023.
\newblock ISSN 2835-8856.

\bibitem[Silver et~al.(2017)Silver, Schrittwieser, Simonyan, Antonoglou, Huang, Guez, Hubert, Baker, Lai, Bolton, et~al.]{silver2017mastering}
David Silver, Julian Schrittwieser, Karen Simonyan, Ioannis Antonoglou, Aja Huang, Arthur Guez, Thomas Hubert, Lucas Baker, Matthew Lai, Adrian Bolton, et~al.
\newblock Mastering the game of go without human knowledge.
\newblock \emph{nature}, 550\penalty0 (7676):\penalty0 354--359, 2017.

\bibitem[Song et~al.(2023)Song, Zhou, Sekhari, Bagnell, Krishnamurthy, and Sun]{song2023hybrid}
Yuda Song, Yifei Zhou, Ayush Sekhari, Drew Bagnell, Akshay Krishnamurthy, and Wen Sun.
\newblock Hybrid {RL}: Using both offline and online data can make {RL} efficient.
\newblock In \emph{The Eleventh International Conference on Learning Representations}, 2023.

\bibitem[Stiennon et~al.(2020)Stiennon, Ouyang, Wu, Ziegler, Lowe, Voss, Radford, Amodei, and Christiano]{stiennon2020learning}
Nisan Stiennon, Long Ouyang, Jeffrey Wu, Daniel Ziegler, Ryan Lowe, Chelsea Voss, Alec Radford, Dario Amodei, and Paul~F Christiano.
\newblock Learning to summarize with human feedback.
\newblock \emph{Advances in Neural Information Processing Systems}, 33:\penalty0 3008--3021, 2020.

\bibitem[Swamy et~al.(2024)Swamy, Dann, Kidambi, Wu, and Agarwal]{swamy2024a}
Gokul Swamy, Christoph Dann, Rahul Kidambi, Steven Wu, and Alekh Agarwal.
\newblock A minimaximalist approach to reinforcement learning from human feedback.
\newblock In \emph{Forty-first International Conference on Machine Learning}, 2024.

\bibitem[Uehara \& Sun(2022)Uehara and Sun]{ueharapessimistic}
Masatoshi Uehara and Wen Sun.
\newblock Pessimistic model-based offline reinforcement learning under partial coverage.
\newblock In \emph{International Conference on Learning Representations}, 2022.

\bibitem[Vaswani(2017)]{vaswani2017attention}
A~Vaswani.
\newblock Attention is all you need.
\newblock \emph{Advances in Neural Information Processing Systems}, 2017.

\bibitem[Von~Stackelberg(2010)]{von2010market}
Heinrich Von~Stackelberg.
\newblock \emph{Market structure and equilibrium}.
\newblock Springer Science \& Business Media, 2010.

\bibitem[Warnell et~al.(2018)Warnell, Waytowich, Lawhern, and Stone]{warnell2018deep}
Garrett Warnell, Nicholas Waytowich, Vernon Lawhern, and Peter Stone.
\newblock Deep tamer: Interactive agent shaping in high-dimensional state spaces.
\newblock In \emph{Proceedings of the AAAI conference on artificial intelligence}, volume~32, 2018.

\bibitem[Wirth et~al.(2017)Wirth, Akrour, Neumann, and F{\"u}rnkranz]{wirth2017survey}
Christian Wirth, Riad Akrour, Gerhard Neumann, and Johannes F{\"u}rnkranz.
\newblock A survey of preference-based reinforcement learning methods.
\newblock \emph{Journal of Machine Learning Research}, 18\penalty0 (136):\penalty0 1--46, 2017.

\bibitem[Wu \& Sun(2024)Wu and Sun]{wu2024making}
Runzhe Wu and Wen Sun.
\newblock Making {RL} with preference-based feedback efficient via randomization.
\newblock In \emph{The Twelfth International Conference on Learning Representations}, 2024.
\newblock URL \url{https://openreview.net/forum?id=Pe2lo3QOvo}.

\bibitem[Xu et~al.(2020)Xu, Wang, Yang, Singh, and Dubrawski]{xu2020preference}
Yichong Xu, Ruosong Wang, Lin Yang, Aarti Singh, and Artur Dubrawski.
\newblock Preference-based reinforcement learning with finite-time guarantees.
\newblock \emph{Advances in Neural Information Processing Systems}, 33:\penalty0 18784--18794, 2020.

\bibitem[Yu et~al.(2020)Yu, Quillen, He, Julian, Hausman, Finn, and Levine]{yu2020meta}
Tianhe Yu, Deirdre Quillen, Zhanpeng He, Ryan Julian, Karol Hausman, Chelsea Finn, and Sergey Levine.
\newblock Meta-world: A benchmark and evaluation for multi-task and meta reinforcement learning.
\newblock In \emph{Conference on robot learning}, pp.\  1094--1100. PMLR, 2020.

\bibitem[Yue et~al.(2012)Yue, Broder, Kleinberg, and Joachims]{yue2012k}
Yisong Yue, Josef Broder, Robert Kleinberg, and Thorsten Joachims.
\newblock The k-armed dueling bandits problem.
\newblock \emph{Journal of Computer and System Sciences}, 78\penalty0 (5):\penalty0 1538--1556, 2012.

\bibitem[Zhan et~al.(2024{\natexlab{a}})Zhan, Uehara, Kallus, Lee, and Sun]{zhan2024provable}
Wenhao Zhan, Masatoshi Uehara, Nathan Kallus, Jason~D. Lee, and Wen Sun.
\newblock Provable offline preference-based reinforcement learning.
\newblock In \emph{The Twelfth International Conference on Learning Representations}, 2024{\natexlab{a}}.

\bibitem[Zhan et~al.(2024{\natexlab{b}})Zhan, Uehara, Sun, and Lee]{zhan2024provablerewardagnostic}
Wenhao Zhan, Masatoshi Uehara, Wen Sun, and Jason~D. Lee.
\newblock Provable reward-agnostic preference-based reinforcement learning.
\newblock In \emph{The Twelfth International Conference on Learning Representations}, 2024{\natexlab{b}}.

\bibitem[Zhu et~al.(2023)Zhu, Jordan, and Jiao]{zhu2023principled}
Banghua Zhu, Michael Jordan, and Jiantao Jiao.
\newblock Principled reinforcement learning with human feedback from pairwise or k-wise comparisons.
\newblock In \emph{International Conference on Machine Learning}, pp.\  43037--43067. PMLR, 2023.

\bibitem[Ziegler et~al.(2019)Ziegler, Stiennon, Wu, Brown, Radford, Amodei, Christiano, and Irving]{ziegler2019fine}
Daniel~M Ziegler, Nisan Stiennon, Jeffrey Wu, Tom~B Brown, Alec Radford, Dario Amodei, Paul Christiano, and Geoffrey Irving.
\newblock Fine-tuning language models from human preferences.
\newblock \emph{arXiv preprint arXiv:1909.08593}, 2019.

\end{thebibliography}
\bibliographystyle{iclr2025_conference}

\newpage
\appendix

\section{Additional Related Work}

\textbf{Empirical PbRL.}
Incorporating preference feedback into reinforcement learning has been explored through various approaches. One standard method involves training a reward model from preferences, which is then used to train a standard RL algorithm~\citep{christiano2017deep,ibarz2018reward}.
A variety of techniques have emerged in this area, including unsupervised pre-training~\citep{lee2021pebble}, exploration driven by uncertainty~\citep{liang2022reward}, data augmentation~\citep{park2022surf}, and meta-learning~\citep{hejna2023few}, to list a few. Another prominent line of research focuses on preference learning via active query methods~\citep{shin2023benchmarks,hwang2024sequential,choi2024listwise}, which have shown strong empirical performance.

Beyond the conventional Markov reward model, some studies have proposed alternative reward structures. For example, \citet{kim2023preference} employed transformer architectures for reward modeling, while \citet{liu2022meta} and \citet{hejna2024inverse} explored learning action-value functions rather than directly modeling rewards. Several approaches bypass explicit reward models altogether, instead optimizing policies directly~\citep{an2023direct,kang2023beyond,hejna2024contrastive}.

\section{Details on the Policy Evaluation Subroutine} \label{sec:PE oracle}

We present a simple policy evaluation subroutine in Algorithm~\ref{alg:PE oracle}. It requires online rollouts and access to the reference policy. The idea of policy evaluation via online rollouts is adopted from \citet{chang2024dataset}, while the analysis follows standard methods.

\begin{algorithm}[!h]
    \caption{\texttt{PE}: Monte Carlo Policy Evaluation} \label{alg:PE oracle}
    \begin{algorithmic}[1]
        \State \textbf{Input:} Reference policy $\piref$, Current policy $\pi^t$, Estimated reward $\rhat$, Number of rollout $K$
        \For{$h \in [H]$}
            \State Collect $K$ i.i.d. trajectories $\{(s^k_1,a^k_1,\dots,s^k_H,a^k_H)\}^K_{k=1}$ 
            \State where $a^k_j \sim \piref_j(\cdot\mid s^k_j)$ for $j < h$, $a^k_h \sim \frac{1}{2}(\piref_h + \pi^t_h)(\cdot\mid s^k_h)$, and $a^k_j \sim \pi^t_j(\cdot\mid s^k_j)$ for $j > h$
            \State Compute $q^k_h = \sum^H_{j = h} \rhat(s^k_j, a^k_j)$, then set $\mathcal{D}^t_h = \{(s^k_h, a^k_h, q^k_h)\}^K_{k=1}$
            \State Least square value function estimation $\Qbar{t}{h} = \argmin_{f \in \mathcal{F}} \frac{1}{K} \sum_{(s,a,q)\in\mathcal{D}^t_h} \left( f(s,a) -  q \right)^2$
        \EndFor
        \State Return $\{\Qbar{t}{h}\}^H_{h=1}$
    \end{algorithmic}
\end{algorithm}

We have the following guarantee.
\begin{lemma} \label{lemma:q value error}
    With probability at least $1-\delta$, Algorithm~\ref{alg:PE oracle} guarantees that, for every $(t,h)\in[T]\times[H]$,
    \begin{align*}
        \mathbb{E}_{s\sim d^{\piref}_{h}, a\sim \frac{1}{2}(\piref^h+\pi^t_h)} \left[ \left(\Qbar{t}{h}(s,a) - \Q{\pi^t}{h, r^t}(s,a)  \right)^2 \right] \leq \frac{c_3 R^2 \log(TH|\mathcal{F}|/\delta)}{K_2} =: \epsilon_{\text{PE}}^2
    \end{align*}
    where $c_3$ is an absolute constant.
\end{lemma}

\begin{proof}
    Since $\norm{\Q{\pi}{h, r}}{\infty} \leq R$ for any policy $\pi$ and $r\in\Rset^H$, Lemma~\ref{lemma:least square} with $B = R$ and $K=K_2$ leads to
    \begin{align*}
        \mathbb{E}_{s\sim d^{\piref}_{h}, a\sim \frac{1}{2}(\piref^h+\pi^t_h)} \left[ \left(\Qbar{t}{h}(s,a) - \Q{\pi^t}{h, r^t}(s,a) \right)^2 \right] \leq \frac{c_3 R^2 \log(|\mathcal{F}|/\delta)}{K_2}
    \end{align*}
    for any fixed $(t,h)\in [T]\times [H]$. The union bound over all $(t,h)\in [T]\times [H]$ concludes the proof.
\end{proof}

\newpage

\section{Theoretical Analysis of \texttt{APPO-rollout}} \label{sec:proof of APPO rollout}

In this section, we provide theoretical analyses of \texttt{APPO-rollout}, a na\"ive algorithm for solving the optimization problem \eqref{eqn:Stackelberg game}. The ideas presented in this section are relevant to the proof of Theorem~\ref{thm:upper bound}, and the results are valuable for comparison with related works.

Before stating the theorem, we define step-wise concentrability, which is always bounded by $\Ctr$. 
\begin{definition}[Step-wise concentrability]
    $\Cstep = \max_{h\in[H]} \sup_{(s,a)\in\Sspace\times\Aspace} \frac{d^{\pistar}_h(s,a)}{d^{\piref}_h(s,a)}$
\end{definition}

\begin{lemma} \label{lem:stepwise concentrability}
    It always holds that $\Cstep \leq \Ctr$.
\end{lemma}

\begin{proof}
    For a fixed pair $(s,a)$, consider the set of trajectories $\mathcal{T}(s,a) := \{\tau = (s_1,a_1,\dots,s_H,a_H) : s_h=s, a_h=a \}$. Then we have that
    \begin{align*}
        d^{\pi}_h(s,a) = \int_{\mathcal{T}(s,a)} d^{\pi}(\tau) d\tau .
    \end{align*}
    for any fixed policy $\pi$. Therefore, for every $(s,a)\in\Sspace\times\Aspace$, we have that
    \begin{align*}
        \frac{d^{\pistar}_h(s,a)}{d^{\piref}_h(s,a)} = \frac{\int_{\mathcal{T}(s,a)} d^{\pistar}(\tau) d\tau}{\int_{\mathcal{T}(s,a)} d^{\piref}(\tau) d\tau} \leq \sup_{\tau} \frac{d^{\pistar}(\tau)}{d^{\piref}(\tau)}  = \Ctr.
    \end{align*}
    Taking the supremum on both sides concludes the proof.
\end{proof}

\begin{theorem} \label{thm:upper bound, rollout}
    Suppose Assumptions~\ref{assumption:reward function class} and \ref{assumption:concentrability} hold. With probability at least $1-\delta$, Algorithm~\ref{alg:APPO rollout} with $\lambda = \Theta(\Ctr), \lambda>\Ctr, \eta = \sqrt{\frac{2\log|\Aspace|}{R^2 T}}$ achieves
    \begin{align*}
        &\V{\pistar}{1, \rstar} - \V{\hat{\pi}}{1, \rstar} \\
        &\leq \mathcal{O}\left( \sqrt{\log\frac{|\Rset|}{\delta}}\left(\frac{\Ctr\kappa\sqrt{H}}{\sqrt{M}} + \frac{R}{\sqrt{K_1}} + \frac{R}{\sqrt{N}}\right) + RH\sqrt{\frac{\log|\Aspace|}{T}} + RH\sqrt{\frac{\Cstep}{K_2}\log\frac{TH|\mathcal{F}|}{\delta}} \right) .
    \end{align*}
    Setting $T = \Theta\left(\frac{R^2H^2 \log |\Aspace| }{\epsilon^2} \right)$, $N = K_1 = \Theta\left( \frac{R^2 \log(|\Rset|/\delta)}{\epsilon^2}\right)$, $M = \Theta\left( \frac{\Ctr^2 \kappa^2 H \log(|\Rset|/\delta)}{\epsilon^2} \right)$, and $K_2 = \Theta\left(\frac{R^2H^2 \Cstep \log(TH|\mathcal{F}|/\delta) }{\epsilon^2} \right)$,
    Algorithm~\ref{alg:APPO} achieves $\epsilon$-optimal policy, i.e. $\V{\pistar}{1, \rstar} - \V{\bar{\pi}}{1, \rstar} \leq \epsilon$.
\end{theorem}

\textbf{Discussion on Theorem~\ref{thm:upper bound, rollout}.}
We compare this bound with PbRL algorithms that assume a known transition model or allow online rollouts.
In comparison to FREEHAND~\citep{zhan2024provable}, $\texttt{APPO-rollout}$ achieves a nearly identical rate for labeled data, but unlike FREEHAND, it requires additional unlabeled trajectories.
This represents a trade-off between statistical efficiency and computational complexity, as FREEHAND relies on solving a nearly intractable nested optimization problem.
Another comparable algorithm is DR-PO~\citep{chang2024dataset}, which establishes a sample complexity of $\Theta\left(\frac{(\Ctr+C_{\text{SFT}})\kappa^2 \log(|\Rset|/\delta)}{\epsilon^2}\right)$ for labeled data.
Unlike \texttt{APPO-rollout}, DR-PO assumes homogeneous rewards, which removes dependence on $H$ from the bound.
While their bound is tighter in $\Ctr$, it comes at the cost of dependence on an additional factor, $C_{\text{SFT}} = \sup_{\pi\in D} \sup_{\tau} \frac{d^{\pi}(\tau)}{d^{\piref}(\tau)}$, where $D$ is a set of policies close to $\piref$ in terms of KL divergence. This additional term arises because DR-PO does not explicitly ensure conservatism.

For simplicity, we introduce some notation regarding optimization objectives in Algorithm~\ref{alg:APPO rollout}. For $r, \tilde{r}\in\Rset^H$, we define
\begin{align*}
    \emplossOpt^t(r; \tilde{r}) &:=  \mathbb{E}_{\tau\sim\dataRollout}\left[ r(\tau) \right] - \mathbb{E}_{\tau\sim\dataTraj}\left[ r(\tau) \right] + \lambda \hat{\mathcal{E}}_{\dataTraj}(r; \tilde{r})
\end{align*}
and its population version as
\begin{align*}
    \lossOpt^t(r; \tilde{r}) &:= \mathbb{E}_{\tau\sim \pi^t} \left[ r(\tau) \right] - \mathbb{E}_{\tau\sim \piref} \left[ r(\tau) \right] + \lambda \mathcal{E}(r; \tilde{r}).
\end{align*}

\subsection{Optimization Error}

In this section, we prove that the (finite-sample) optimization objective $\emplossOpt^t(r; \rhat)$ is close to its population version, $\lossOpt^t(r; \tilde{r})$. The result ensures that $r^t$ is a good approximation of the solution to the optimization program with infinite samples, i.e.
\begin{align*}
    r^t \approx \argmin_{r\in\Rset^H} \lossOpt^t(r; \rhat).
\end{align*}

\begin{lemma} \label{lemma:optimization error, rollout}
    With probability at least $1-\delta/2$, for all $t\in [T]$, we have
    \begin{align*}
        \lossOpt(r^t; \rhat) \leq \lossOpt(\rstar; \rhat) + 2\tilde{\epsilon}_{approx}
    \end{align*}
    where $\tilde{\epsilon}_{approx}$ is defined in Lemma~\ref{lemma:closeness of loss, rollout}.
\end{lemma}

\begin{proof}
    We have the following decomposition:
    \begin{align*}
        &\lossOpt^t(r^t ; \rhat) - \lossOpt^t(\rstar; \rhat) \\
        &= \underbrace{ \lossOpt^t(r^t ; \pi^t) - \emplossOpt^t(r^t; \rhat)}_{\text{(I)}}
        + \underbrace{\emplossOpt^t(r^t; \rhat) - \emplossOpt^t(\rstar; \rhat) }_{\text{(II)}}
        + \underbrace{ \emplossOpt^t(\rstar; \rhat) - \lossOpt^t(\rstar; \rhat) }_{\text{(III)}}
    \end{align*}
    Conditioned on the event defined by Lemma~\ref{lemma:closeness of loss}, (I) and (III) are bounded by $\epsilon_{opt}$. Moreover, the optimality of $r^t$ implies (II)$\leq 0$.
\end{proof}

\begin{lemma} \label{lemma:closeness of loss, rollout}
    With probability at least $1-\delta/2$, for every $t\in [T]$ and $r\in \Rset^H$, it holds that
    \begin{align*}
        \left| \lossOpt^t(r; \rhat) - \emplossOpt^t(r; \rhat) \right| \leq
        R\sqrt{\frac{\log(6|\Rset|/\delta)}{2K_1}} + 2R\sqrt{\frac{2\log(6|\Rset|/\delta)}{N}} := \tilde{\epsilon}_{approx}
    \end{align*}
\end{lemma}

\begin{proof}
    
    Fix $r \in \Rset^H$, and note that
    \begin{align*}
        &\left| \lossOpt^t(r; \rhat) - \emplossOpt^t(r; \rhat) \right| \\
        &\leq \left| \mathbb{E}_{\tau\sim\dataRollout^t}[r(\tau)] - \mathbb{E}_{\tau\sim\pi^t}[r(\tau)] \right|
        + \left| \mathbb{E}_{\tau\sim\dataTraj}[r(\tau)] - \mathbb{E}_{\tau\sim\piref}[r(\tau)] \right| \\
        &\quad + \left| \mathbb{E}_{(\tau^0,\tau^1)\sim\dataTraj}\left[(r-\rhat)(\tau^0) -(r-\rhat)(\tau^1)\right] - \mathbb{E}_{(\tau^0,\tau^1)\sim\piref}\left[(r-\rhat)(\tau^0) -(r-\rhat)(\tau^1)\right] \right|.
    \end{align*}
    Since $|r(\tau)|\leq R$ and $|(r-\rhat)(\tau)|\leq R$ for any trajectory $\tau$, each term can be bounded by Hoeffding inequality. Specifically, each of these three events occurs with probability at least $1-\delta/6$:
    \begin{align*}
        &\left| \mathbb{E}_{\tau\sim\dataRollout^t}[r(\tau)] - \mathbb{E}_{\tau\sim\pi^t}[r(\tau)] \right| \leq R\sqrt{\frac{\log(6/\delta)}{2K_1}}, \\
        &\left| \mathbb{E}_{\tau\sim\dataTraj}[r(\tau)] - \mathbb{E}_{\tau\sim\piref}[r(\tau)] \right| \leq R\sqrt{\frac{\log(6/\delta)}{2N}}, \\
        &\left| \mathbb{E}_{(\tau^0,\tau^1)\sim\dataTraj}\left[(r-\rhat)(\tau^0) -(r-\rhat)(\tau^1)\right] - \mathbb{E}_{(\tau^0,\tau^1)\sim\piref}\left[(r-\rhat)(\tau^0) -(r-\rhat)(\tau^1)\right] \right|
        \leq 2R\sqrt{\frac{\log(6/\delta)}{2N}}.
    \end{align*}
    Taking union bound over these events and all $r\in \Rset^H$, with probability at least $1-\delta/2$, it holds that
    \begin{align*}
        \left| \lossOpt^t(r; \rhat) - \emplossOpt^t(r; \rhat) \right| &\leq R\sqrt{\frac{\log(6|\Rset|/\delta)}{2K_1}} + R\sqrt{\frac{\log(6|\Rset|/\delta)}{2N}} + 2R\sqrt{\frac{\log(6|\Rset|/\delta)}{2N}} \\
        &\leq R\sqrt{\frac{\log(6|\Rset|/\delta)}{2K_1}} + 2R\sqrt{\frac{2\log(6|\Rset|/\delta)}{N}}
    \end{align*}
    for every $r\in \Rset^H$.
\end{proof}

\subsection{Policy Update} \label{sec:rollout NPG analysis}

We present the guarantee regarding the policy update steps. The proofs in this section are based on the standard analysis of the natural policy gradient (also referred to as trust region policy optimization)~\citep{cai2020provably,chang2024dataset}.

\begin{lemma} \label{lemma:NPG regret, rollout}
    With probability at least $1-\delta/4$, it holds that
    \begin{align*}
        \frac{1}{T} \sum^T_{t=1} \left( \V{\pistar}{1, r^t}(s_1) - \V{\pi^t}{1, r^t}(s_1) \right) \leq RH\sqrt{\frac{\log|\Aspace|}{2T}} + 2H \epsilon_{\text{PE}} \sqrt{ 2 \Cstep } 
    \end{align*}
\end{lemma}

\begin{proof}[Proof of Lemma~\ref{lemma:NPG regret, rollout}]
    
    The performance difference lemma (Lemma~\ref{lemma:performance difference lemma}) implies that
    \begin{align*}
        &\sum^T_{t=1} \left( \V{\pistar}{1, r^t}(s_1) - \V{\pi^t}{1, r^t}(s_1) \right)  \\
        &= \sum^T_{t=1}  \mathbb{E}_{\pistar}\left[ \sum^H_{h=1} \langle \Q{\pi^t}{h, r^t}(s_h,\cdot), \pistar_h(\cdot\mid s_h) - \pi^t_h(\cdot\mid s_h) \rangle \right] \\
        &= \underbrace{ \sum^T_{t=1} \sum^H_{h=1}  \mathbb{E}_{s\sim d^{\pistar}_h}\left[ \langle \Qbar{t}{h}(s,\cdot), \pistar_h(\cdot\mid s) - \pi^t_h(\cdot\mid s) \rangle \right] }_{\text{(I)}} \\
        &\quad + \underbrace{ \sum^T_{t=1} \sum^H_{h=1} \mathbb{E}_{s\sim d^{\pistar}_h}\left[ \langle (\Q{\pi^t}{h, r^t}-\Qbar{t}{h})(s,\cdot), \pistar_h(\cdot\mid s_h) - \pi^t_h(\cdot\mid s_h) \rangle \right] }_{\text{(II)}}
    \end{align*}

    \paragraph{Bounding (I).}
    Decompose the inner product inside the expectation:
    \begin{align}
        &\langle \eta \Qbar{t}{h}(s_h,\cdot), \pistar_h(\cdot\mid s) - \pi^t_h(\cdot\mid s) \rangle \notag \\
        &\langle \eta \Qbar{t}{h}(s_h,\cdot), \pistar_h(\cdot\mid s) - \pi^{t+1}_h(\cdot\mid s) \rangle + \langle \eta \Qbar{t}{h}(s_h,\cdot), \pi^{t+1}_h(\cdot\mid s) - \pi^{t}_h(\cdot\mid s) \rangle \notag \\
        &\leq \langle \eta \Qbar{t}{h}(s_h,\cdot), \pistar_h(\cdot\mid s) - \pi^{t+1}_h(\cdot\mid s) \rangle + \eta \norm{\Qbar{t}{h}(s_h,\cdot)}{\infty} \norm{\pistar_h(\cdot\mid s) - \pi^{t+1}_h(\cdot\mid s)}{1} \notag \\
        &\leq \langle \eta \Qbar{t}{h}(s_h,\cdot), \pistar_h(\cdot\mid s) - \pi^{t+1}_h(\cdot\mid s) \rangle + \eta R \norm{\pistar_h(\cdot\mid s) - \pi^{t+1}_h(\cdot\mid s)}{1} \label{eqn:NPG regret, rollout 1}
    \end{align}
    where we use H\"older's inequality with the fact that $\norm{\Qbar{t}{h}}{\infty}\leq R$. Now recall that the policy update step (Line 7) in Algorithm~\ref{alg:APPO rollout} leads to
    \begin{align*}
        \pi^{t+1}_h(\cdot\mid s) = \frac{1}{Z^t_h(s)} \pi^t_h(\cdot\mid s) \exp\left( \eta \Qbar{t}{h}(s,\cdot) \right)
    \end{align*}
    where $Z^t_h(s) = \sum_{a\in\Aspace} \pi^t_h(a \mid s) \exp\left( \eta \Qbar{t}{h}(s,a) \right)$. Using the relationship $\eta\Qbar{t}{h}(s,a) = \log Z^t_h(s) + \log \pi^{t+1}_h(a\mid s) - \log \pi^{t}_h(a\mid s)$, it holds that
    
    \begin{align*}
        &\langle \eta \Qbar{t}{h}(s_h,\cdot), \pistar_h(\cdot\mid s) - \pi^{t+1}_h(\cdot\mid s) \rangle \\
        &= \langle \log Z^t_h(s) + \log \pi^{t+1}_h(\cdot\mid s) - \log \pi^{t}_h(\cdot\mid s), \pistar_h(\cdot\mid s) - \pi^{t+1}_h(\cdot\mid s) \rangle \\
        &= \langle  \log \pi^{t+1}_h(\cdot\mid s) - \log \pi^{t}_h(\cdot\mid s), \pistar(\cdot\mid s) - \pi^{t+1}_h(\cdot\mid s) \rangle \\
        &= \langle  \log \pi^{t+1}_h(\cdot\mid s) - \log \pi^{t}_h(\cdot\mid s), \pistar(\cdot\mid s) \rangle - \kldiv{\pi^{t+1}_h(\cdot\mid s)}{\pi^{t}_h(\cdot\mid s)} \\
        &= \langle  \log \frac{\pi^{t+1}_h(\cdot\mid s)}{\pistar_h(\cdot\mid s)} +  \log \frac{\pistar_h(\cdot\mid s)}{\pi^{t}_h(\cdot\mid s)}, \pistar_h(\cdot\mid s) \rangle - \kldiv{\pi^{t+1}_h(\cdot\mid s)}{\pi^{t}_h(\cdot\mid s)} \\
        &= \kldiv{\pistar_h(\cdot\mid s)}{\pi^t_h(\cdot\mid s)} - \kldiv{\pistar_h(\cdot\mid s)}{\pi^{t+1}_h(\cdot\mid s)} - \kldiv{\pi^{t+1}_h(\cdot\mid s)}{\pi^{t}_h(\cdot\mid s)} \\
        &\leq \kldiv{\pistar_h(\cdot\mid s)}{\pi^t_h(\cdot\mid s)} - \kldiv{\pistar_h(\cdot\mid s)}{\pi^{t+1}_h(\cdot\mid s)} - \frac{1}{2}\norm{\pistar_h(\cdot\mid s) - \pi^{t+1}_h(\cdot\mid s)}{1}^2
    \end{align*}
    where the second equality holds since $Z^t_h(s)$ is a constant given $s$, and the last inequality holds due to Pinsker's inequality. Combining this bound with \eqref{eqn:NPG regret, rollout 1}, we obtain
    
    \begin{align*}
        & \sum^T_{t=1} \langle \eta \Qbar{t}{h}(s_h,\cdot), \pistar_h(\cdot\mid s) - \pi^t_h(\cdot\mid s) \rangle \\
        &= \sum^T_{t=1} \left( \kldiv{\pistar_h(\cdot\mid s)}{\pi^t_h(\cdot\mid s)} - \kldiv{\pistar_h(\cdot\mid s)}{\pi^{t+1}_h(\cdot\mid s)} \right) \\
        &\quad + \sum^T_{t=1} \left( \eta R \norm{\pistar_h(\cdot\mid s) - \pi^{t+1}_h(\cdot\mid s)}{1} - \frac{1}{2}\norm{\pistar_h(\cdot\mid s) - \pi^{t+1}_h(\cdot\mid s)}{1}^2 \right) \\
        &\leq \sum^T_{t=1} \left( \kldiv{\pistar_h(\cdot\mid s)}{\pi^t_h(\cdot\mid s)} - \kldiv{\pistar_h(\cdot\mid s)}{\pi^{t+1}_h(\cdot\mid s)} \right) + \sum^T_{t=1} \frac{\eta^2 R^2}{2} \\
        &= \kldiv{\pistar_h(\cdot\mid s)}{\pi^1_h(\cdot\mid s)} - \kldiv{\pistar_h(\cdot\mid s)}{\pi^{T+1}_h(\cdot\mid s)} + \frac{\eta^2 R^2 T}{2} \\
        &\leq \log|\Aspace| + \frac{\eta^2 R^2 T}{2}
    \end{align*}
    where the first inequality holds since $\forall x\in\Real \, ax - x^2/2 \leq a^2/2$, and the second inequality holds due to the fact that $\pi^1_h = \text{Unif}(\Aspace)$. Finally, setting $\eta = \sqrt{\frac{2\log|\Aspace|}{R^2 T}}$, (I) is bounded by
    
    \begin{align*}
        \text{(I)} &= \sum^H_{h=1}  \mathbb{E}_{s\sim d^{\pistar}_h}\left[ \sum^T_{t=1}  \langle \Qbar{t}{h}(s,\cdot), \pistar(\cdot\mid s) - \pi^t(\cdot\mid s) \rangle \right] \\
        &\leq \sum^H_{h=1} \frac{\log|\Aspace|}{\eta} + \frac{\eta R^2 T}{2} = RH\sqrt{\frac{T\log|\Aspace|}{2}}
    \end{align*}

    \paragraph{Bounding (II).}
    We condition on the event defined by Lemma~\ref{lemma:q value error}. Then we have
    \begin{align*}
        &\left| \mathbb{E}_{s\sim d^{\pistar}_h}\left[ \langle (\Q{\pi^t}{h, r^t}-\Qbar{t}{h})(s,\cdot), \pistar_h \rangle \right] \right| \\
        &= \left| \mathbb{E}_{(s,a)\sim d^{\pistar}_h}\left[ \Q{\pi^t}{h, r^t}(s,a)-\Qbar{t}{h}(s,a)\right] \right| \\
        &\leq \sqrt{ \mathbb{E}_{(s,a)\sim d^{\pistar}_h}\left[ \left(\Q{\pi^t}{h, r^t}(s,a)-\Qbar{t}{h}(s,a)\right)^2 \right] } \\
        &\leq \sqrt{ 2 \left( \max_{h\in[H]}\sup_{(s,a)\in\Sspace\times\Aspace}\frac{d^{\pistar}_h(s,a)}{d^{\piref}_h(s,a)} \right) \mathbb{E}_{s\sim d^{\piref}_h, a\sim \frac{1}{2}(\pi^t_h+\piref_h)}\left[ \left(\Q{\pi^t}{h, r^t}(s,a)-\Qbar{t}{h}(s,a)\right)^2 \right] } \\
        &\leq \sqrt{ 2 \Cstep \epsilon_{\text{PE}}^2 }
    \end{align*}
    where the first inequality holds due to Jensen's inequality, the second inequality uses importance sampling, and the last inequality uses Lemma~\ref{lemma:q value error}.

    \begin{align*}
        &\left| \mathbb{E}_{s\sim d^{\pistar}_h}\left[ \langle (\Q{\pi^t}{h, r^t}-\Qbar{t}{h})(s,\cdot), \pi^t_h \rangle \right] \right| \\
        &= \left| \mathbb{E}_{s\sim d^{\pistar}_h, a\sim \pi^t_h}\left[ \Q{\pi^t}{h, r^t}(s,a)-\Qbar{t}{h}(s,a)\right] \right| \\
        &\leq \sqrt{ \mathbb{E}_{s\sim d^{\pistar}_h, a\sim \pi^t_h}\left[  \left(\Q{\pi^t}{h, r^t}(s,a)-\Qbar{t}{h}(s,a)\right)^2 \right] } \\
        &\leq \sqrt{ 2 \left( \max_{h\in[H]}\sup_{s\in\Sspace}\frac{d^{\pistar}_h(s)}{d^{\piref}_h(s)} \right) \mathbb{E}_{s\sim d^{\piref}_h, a\sim \frac{1}{2}(\pi^t_h+\piref_h)}\left[ \left(\Q{\pi^t}{h, r^t}(s,a)-\Qbar{t}{h}(s,a)\right)^2 \right] } \\
        &\leq \sqrt{ 2 \left( \max_{h\in[H]}\sup_{s\in\Sspace}\frac{d^{\pistar}_h(s)}{d^{\piref}_h(s)} \right) \mathbb{E}_{s\sim d^{\piref}_h, a\sim \frac{1}{2}(\pi^t_h+\piref_h)}\left[ \left(\Q{\pi^t}{h, r^t}(s,a)-\Qbar{t}{h}(s,a)\right)^2 \right] } \\
        &\leq \sqrt{ 2 \Cstep \epsilon_{\text{PE}}^2 }.
    \end{align*}
    Therefore, we obtain the bound
    \begin{align*}
        \text{(II)} &\leq \sum^T_{t=1}\sum^H_{h=1} \left|\mathbb{E}_{s\sim d^{\pistar}_h}\left[ \langle (\Q{\pi^t}{h, r^t}-\Qbar{t}{h})(s,\cdot), \pistar_h(\cdot\mid s_h) \rangle \right] \right| \\
        &\quad + \sum^T_{t=1}\sum^H_{h=1} \left|\mathbb{E}_{s\sim d^{\pistar}_h}\left[ \langle (\Q{\pi^t}{h, r^t}-\Qbar{t}{h})(s,\cdot), \pi^t_h(\cdot\mid s_h) \rangle \right] \right| \\
        &\leq 2TH \epsilon_{\text{PE}} \sqrt{ 2 \Cstep }.
    \end{align*}

    We conclude the proof by combining the bounds on (I) and (II).
\end{proof}

Now we prove Theorem~\ref{thm:upper bound, rollout} based on the lemmas.

\begin{proof}[Proof of Theorem~\ref{thm:upper bound, rollout}]
    We condition on the event defined by Lemma~\ref{lemma:reward MLE error} (with $\delta'=\delta/4$), Lemma~\ref{lemma:optimization error, rollout}, and Lemma~\ref{lemma:NPG regret, rollout}, which hold simultaneously with probability at least $1-\delta$. Consider the following sub-optimality decomposition at step $t$:
    \begin{align}
        \V{\pistar}{1, \rstar} - \V{\pi^t}{1, \rstar} \notag
        &= \V{\pistar}{1, \rstar} - \V{\pistar}{1, \rhat} + \V{\pistar}{1, \rhat} - \V{\pistar}{1, r^t} + \V{\pistar}{1, r^t} - \V{\pi^t}{1, \rstar} + \V{\pi^t}{1, r^t} - \V{\pi^t}{1, r^t} \notag \\
        &= \underbrace{\V{\pistar}{1,\rstar-\rhat} - \V{\piref}{1,\rstar-\rhat}}_{\text{(I) : MLE estimation error}} \notag \\
        &\quad + \underbrace{ \V{\pistar}{1,\rhat-r^t} - \V{\piref}{1,\rhat-r^t} - \V{\pi^t}{1,\rstar} + \V{\piref}{1, \rstar} + \V{\pi^t}{1,r^t} - \V{\piref}{1, r^t}}_{\text{(II) : Optimization error}} \notag \\
        &\quad + \underbrace{\V{\pistar}{1, r^t} - \V{\pi^t}{1, r^t} }_{\text{(III) : Policy update regret}},  \label{eqn:upper bound, rollout 1}
    \end{align}
    where we omit the initial state $s_1$ for simplicity.

    \paragraph{Bounding (I).}
    Since we condition on the event defined by Lemma~\ref{lemma:reward MLE error}, we have
    \begin{align*}
        \text{(I)} &= \V{\pistar}{1,\rstar-\rhat} - \V{\piref}{1,\rstar-\rhat} \\
        &= \mathbb{E}_{\tau^0\sim \pistar, \tau^1 \sim \piref} \left[ \rstar(\tau^0) - \rstar(\tau^1) - \rhat(\tau^0) + \rhat(\tau^1)  \right] \\
        &\leq \sqrt{\mathbb{E}_{\tau^0\sim \pistar, \tau^1 \sim \piref} \left[ | \rstar(\tau^0) - \rstar(\tau^1) - \rhat(\tau^0) + \rhat(\tau^1) |^2  \right]} \\
        &\leq \sqrt{\Ctr \mathbb{E}_{\tau^0, \tau^1 \sim \piref} \left[ | \rstar(\tau^0) - \rstar(\tau^1) - \rhat(\tau^0) + \rhat(\tau^1) |^2  \right]} \\
        &\leq \sqrt{\Ctr} \epsilon_{r}(\delta/4).
    \end{align*}

    \paragraph{Bounding (II).}
    We can relate the terms $\V{\pistar}{1,\rhat-r^t} - \V{\piref}{1,\rhat-r^t}$ to $\mathcal{E}(r^t; \pstar, \rhat)$. By Assumption~\ref{assumption:concentrability}, we have that
    \begin{align*}
       &\V{\pistar}{1,\rhat-r^t} - \V{\piref}{1,\rhat-r^t}  \\
        &= \mathbb{E}_{\tau^0\sim \pistar, \tau^1 \sim \piref} \left[ \rhat(\tau^0) - \rhat(\tau^1) - r^t(\tau^0) + r^t(\tau^1)  \right]  \\
        &\leq \Ctr \mathbb{E}_{\tau^0, \tau^1 \sim \piref} \left[ | \rhat(\tau^0) - \rhat(\tau^1) - r^t(\tau^0) + r^t(\tau^1) |  \right] \\
        &= \Ctr \mathcal{E}(r^t; \rhat) \leq \lambda \mathcal{E}(r^t; \rhat)
    \end{align*}
    where the last inequality holds since $\mathcal{E}(r^t; \rhat)$ is non-negative and $\lambda \geq \Ctr$. Further, Lemma~\ref{lemma:optimization error, rollout} implies
    \begin{align*}
        \lambda \mathcal{E}(r^t; \rhat) &\leq \V{\pi^t}{1, \rstar} - \V{\piref}{1, \rstar} - \V{\pi^t}{1, r^t} + \V{\piref}{1, r^t} + \lambda \mathcal{E}(\rstar; \rhat) + 2\tilde{\epsilon}_{approx} \\
        &\leq \V{\pi^t}{1, \rstar} - \V{\piref}{1, \rstar} - \V{\pi^t}{1, r^t} + \V{\piref}{1, r^t} + \lambda \tilde{\epsilon}_{r}(\delta/4) + 2\tilde{\epsilon}_{approx}
    \end{align*}
    where the last inequality holds due to Lemma~\ref{lemma:reward MLE error}:
    \begin{align*}
        \mathcal{E}(\rstar; \rhat) &= \mathbb{E}_{\tau^0, \tau^1 \sim \piref} \left[ | \rhat(\tau^0) - \rhat(\tau^1) - \rstar(\tau^0) + \rstar(\tau^1) |  \right] \\
        &\leq \sqrt{\mathbb{E}_{\tau^0, \tau^1 \sim \piref} \left[ | \rhat(\tau^0) - \rhat(\tau^1) - \rstar(\tau^0) + \rstar(\tau^1) |^2 \right]} \leq \epsilon_{r}(\delta/4).
    \end{align*}
    Therefore, we have
    \begin{align*}
        \text{(II)} \leq \lambda \epsilon_{r}(\delta/4) + 2\tilde{\epsilon}_{approx}.
    \end{align*}

    \paragraph{Bounding Sub-optimality.}
    Putting the bounds on (I) and (II) into \eqref{eqn:upper bound, rollout 1}, we have
    \begin{align}
        &\V{\pistar}{1, \rstar} - \V{\pi^t}{1, \rstar} \notag \\
        &\leq \sqrt{\Ctr} \epsilon_{r}(\delta/4) + \lambda \epsilon_r(\delta/4) + 2\tilde{\epsilon}_{approx} + \V{\pistar}{1, r^t} - \V{\pi^t}{1, r^t} \label{eqn:upper bound, rollout 2}
    \end{align}
    
    Since Algorithm~\ref{alg:APPO rollout} returns the mixture policy $\hat{\pi} = \frac{1}{T}\sum^T_{t=1} \pi^t$, the sub-optimality is $\V{\pistar}{1, \rstar} - \V{\hat{\pi}}{1, \rstar} = \frac{1}{T} \sum^T_{t=1} \left( \V{\pistar}{1, \rstar} - \V{\pi^t}{1, \rstar} \right)$. Using the bound in \eqref{eqn:upper bound, rollout 2} and Lemma~\ref{lemma:NPG regret, rollout}, it holds that
    
    \begin{align*}
        &\V{\pistar}{1, \rstar} - \V{\hat{\pi}}{1, \rstar} \\
        &= \frac{1}{T} \sum^T_{t=1} \left( \V{\pistar}{1, \rstar} - \V{\pi^t}{1, \rstar} \right) \\
        &\leq \sqrt{\Ctr} \epsilon_{r}(\delta/4) + \lambda \epsilon_r(\delta/4) + 2\tilde{\epsilon}_{approx} + \frac{1}{T} \sum^T_{t=1} \left( \V{\pistar}{1, r^t} - \V{\pi^t}{1, r^t} \right) \\
        &\leq \sqrt{\Ctr} \epsilon_{r}(\delta/4) + \lambda \epsilon_r(\delta/4) + 2\tilde{\epsilon}_{approx} + RH\sqrt{\frac{\log|\Aspace|}{2T}} + 2H\epsilon_{\text{PE}}\sqrt{\Cstep} \\
        &\leq \mathcal{O}\left( \sqrt{\log\frac{|\Rset|}{\delta}}\left(\frac{\Ctr\kappa\sqrt{H}}{\sqrt{M}} + \frac{R}{\sqrt{K_1}} + \frac{R}{\sqrt{N}}\right) + RH\sqrt{\frac{\log|\Aspace|}{T}} + RH\sqrt{\frac{\Cstep}{K_2}\log\frac{TH|\mathcal{F}|}{\delta}} \right)
    \end{align*}
    
\end{proof}

\newpage

\section{Detailed Proof of Theorem~\ref{thm:upper bound}} \label{sec:proof of APPO}

For simplicity, we introduce some notations regarding optimization objectives in Algorithm~\ref{alg:APPO}. For $f\in\mathcal{F}^H$, we define
\begin{align*}
    \emplossOpt^t(f; \tilde{P}, \tilde{r}) &:= \sum^H_{h=1} \E_{(s_h,a_h)\sim\dataTraj}\left[ f_h\circ\pi^t_h (s_h) - f_h(s_h, a_h) \right] + \lambda \hat{\mathcal{E}}_{\dataTraj}(f; \tilde{P}, \tilde{r}) 
\end{align*}
and its population version as
\begin{align*}
    \lossOpt^t(f; \tilde{P}, \tilde{r}) &:= \sum^H_{h=1} \mathbb{E}_{(s_h,a_h)\sim d^{\piref}_h}\left[ f_h\circ\pi^t_h (s_h) - f_h(s_h, a_h) \right] + \lambda \mathcal{E}(f; \tilde{P}, \tilde{r})
\end{align*}

\subsection{Optimization Error}

In this section, we prove that the (finite-sample) optimization objective $\emplossOpt^t(f; \phat, \rhat)$ is close to its population version $\lossOpt^t(f; \tilde{P}, \tilde{r})$. The result ensures that $f^t$ is a good approximation for the solutions to the optimization program with infinite samples, i.e.
\begin{align*}
   f^t \approx \argmin_{f\in\Fset^H} \lossOpt^t(f; \pstar, \rhat).
\end{align*}

\textbf{Remark.}
For this section, we assume that the maximum likelihood transition estimator $\phat$ is computed using half of $\dataTraj$, and the losses $\emplossOpt^t(f; \phat, \rhat)$ are computed from the other half. This increases the sample complexity only by a constant factor but helps avoid union bound over $\Pset$ in the proof of Lemma~\ref{lemma:closeness of loss}.

\begin{lemma} \label{lemma:optimization error}
    With probability at least $1-\delta/2$, for all $t\in [T]$, we have that
    \begin{align*}
        \lossOpt^t(f^t; \rhat) \leq \lossOpt^t(\Q{\pi^t}{} ; \rhat) + 2\epsilon_{approx}
    \end{align*}
    where $\epsilon_{approx}$ is defined in Lemma~\ref{lemma:closeness of loss}.
\end{lemma}

\begin{proof}
    Consider this decomposition: 
    \begin{align*}
        &\lossOpt^t(f^t ;\rhat) - \lossOpt^t(\Q{\pi^t}{} ;  \rhat) \\
        &= \underbrace{ \lossOpt^t(f^t ; \rhat) - \emplossOpt^t(f^t ; \phat, \rhat)}_{\text{(I)}} 
        + \underbrace{\emplossOpt^t(f^t ; \phat, \rhat) - \emplossOpt^t(\Q{\pi^t}{} ; \phat, \rhat)}_{\text{(II)}} 
        + \underbrace{ \emplossOpt^t(\Q{\pi^t}{} ; \phat, \rhat) - \lossOpt^t(\Q{\pi^t}{} ; \rhat) }_{\text{(III)}}.
    \end{align*}
    Conditioned on the event defined by Lemma~\ref{lemma:closeness of loss}, (I) and (III) are bounded by $\epsilon_{approx}$. Moreover, the optimality of $f^t$ implies (II)$\leq 0$.
\end{proof}

\begin{lemma} \label{lemma:closeness of loss}
    With probability at least $1-\delta/2$, for every $t\in[T]$ and $f\in \mathcal{F}^H$, it holds that
    \begin{align*}
        \left| \emplossOpt^t(f; \phat, \rhat) - \lossOpt^t(f; \rhat) \right| \leq
        8R\sqrt{\frac{H^3T\log(8H|\mathcal{F}|/\delta)}{N}} + 2RH \epsilon_P(\delta/8) := \epsilon_{approx}.
    \end{align*}
\end{lemma}

\begin{proof}
    Due to the policy update in Line 7 of Algorithm~\ref{alg:APPO}, the policies $\{\pi^t_h\}_{(t,h)\in[T]\times[H]}$ belongs to the following function class:
    \begin{align*}
        \Pi = \left\{ \pi(a\mid s) =  \frac{ \exp\left(\sum^T_{i=1} \eta f^i(s,a)\right) }{ \sum_{a'\in\Aspace} \exp\left(\sum^T_{i=1} \eta f^i(s,a') \right) } : f^i \in \mathcal{F} \text{ for all } i\in[T] \right\}.
    \end{align*}
    It is clear that $|\Pi| \leq |\mathcal{F}|^T $.

    \paragraph{Step 1.}
    Fix $h\in[H]$, $f \in\mathcal{F}$, and $\pi\in \Pi$. Since $|f\circ \pi(s)|\leq R$ for all $s\in\Sspace$, Hoeffding inequality implies that
    \begin{align*}
        \left| \E_{s_h\in\dataTraj}\left[ f\circ\pi (s_h)\right] - \mathbb{E}_{s_h\sim d^{\piref}_h}\left[ f\circ\pi (s_h) \right] \right| \leq R\sqrt{\frac{\log(8/\delta)}{2N}}
    \end{align*}
    with probability at least $1-\delta/8$. Similarly, since $|f(s,a)|\leq R$ for all $(s,a)\in\Sspace\times\Aspace$, it holds that
    \begin{align*}
        \left| \E_{(s_h,a_h)\sim\dataTraj}\left[ f(s_h,a_h)\right] - \mathbb{E}_{(s_h,a_h)\sim d^{\piref}_h}\left[ f(s_h,a_h) \right] \right| \leq R\sqrt{\frac{\log(8/\delta)}{2N}}
    \end{align*}
    with probability at least $1-\delta/8$. Thus, with probability at least $1-\delta/4$, we have
    \begin{align*}
        &\left|\mathbb{E}_{(s_h,a_h)\sim\dataTraj}\left[ f\circ\pi (s_h) - f(s_h,a_h)\right] - \mathbb{E}_{(s_h,a_h)\sim d^{\piref}_h}\left[ f\circ\pi(s_h) - f(s_h,a_h) \right] \right| \\
        &\leq  \left| \mathbb{E}_{(s_h,a_h)\sim\dataTraj}\left[ f\circ\pi (s_h)\right] - \mathbb{E}_{(s_h,a_h)\sim d^{\piref}_h}\left[ f\circ\pi(s_h) \right] \right| \\
        &\quad + \left| \mathbb{E}_{(s_h,a_h)\sim\dataTraj}\left[ f (s_h,a_h)\right] - \mathbb{E}_{(s_h,a_h)\sim d^{\piref}_h}\left[ f(s_h,a_h) \right] \right| \\
        &\leq R\sqrt{\frac{2\log(8/\delta)}{N}}.
    \end{align*}
    Consider union bound over all $h\in[H], f\in\mathcal{F}$, and $\pi\in\Pi$. Since $\pi^t_h \in\Pi$ for every $(t,h)\in[T]\times[H]$, with probability at least $1-\delta/4$, we have 
    \begin{align*}
        &\left|\sum^H_{h=1} \mathbb{E}_{(s_h,a_h)\sim\dataTraj}\left[ f_h\circ\pi^t_h (s_h) - f_h(s_h,a_h)\right] - \sum^H_{h=1} \mathbb{E}_{(s_h,a_h)\sim d^{\piref}_h}\left[ f_h\circ\pi^t_h(s_h) - f_h(s_h,a_h) \right] \right| \\
        &\leq \sum^H_{h=1} \left| \mathbb{E}_{(s_h,a_h)\sim\dataTraj}\left[ f_h\circ\pi^t_h (s_h) - f_h(s_h,a_h)\right] - \mathbb{E}_{(s_h,a_h)\sim d^{\piref}_h}\left[ f_h\circ\pi^t_h(s_h) - f_h(s_h,a_h) \right] \right| \\
        &\leq RH\sqrt{\frac{2\log(8H|\mathcal{F}||\Pi|/\delta)}{N}} 
        \leq 2RH\sqrt{\frac{T\log(8H|\mathcal{F}|/\delta)}{N}}.
    \end{align*}
    for every $f\in\mathcal{F}$.

    \paragraph{Step 2.} We have that
    \begin{align}\label{eqn:closeness of loss 1}
        |\hat{\mathcal{E}}_{\dataTraj}(f; \phat, \rhat) - \mathcal{E}(f; \pstar, \rhat) | 
        \leq |\hat{\mathcal{E}}_{\dataTraj}(f; \phat, \rhat) - \mathcal{E}(f; \phat, \rhat) | + |\mathcal{E}(f; \phat, \rhat) - \mathcal{E}(f; \pstar, \rhat) |.
    \end{align}
    
    Again, we use Hoeffding inequality to bound the first term. Fix $f \in \mathcal{F}^H$ and $\pi = \{\pi_h\}^H_{h=1} \in \Pi^H$ and consider the function $r^{\pi}_{\phat, f}$ (Recall that $r^{\pi}_{h, \phat, f} (s,a) = f_h(s,a) - \phat (f_{h+1}\circ \pi_{h+1})(s,a)$ for all $h\in[H]$ and $(s,a)\in\Sspace\times\Aspace$). Since $|(r^{\pi}_{\phat, f} - \rhat)(\tau)| \leq 2RH$ for any trajectory $\tau$, we have that
    \begin{align*}
        &\left| \mathbb{E}_{(\tau^0,\tau^1)\sim\dataTraj}\left[ \left|(r^{\pi}_{\phat, f} - \rhat)(\tau^0) - (r^{\pi}_{\phat, f} - \rhat)(\tau^1) \right| \right] - \mathbb{E}_{(\tau^0,\tau^1)\sim \piref}\left[ \left|(r^{\pi}_{\phat, f} - \rhat)(\tau^0) - (r^{\pi}_{\phat, f} - \rhat)(\tau^1)\right| \right] \right| \\
        &\leq 2RH \sqrt{\frac{2\log(8/\delta)}{N}}
    \end{align*}
    with probability at least $1-\delta/8$. Applying union bound over all $f\in\mathcal{F}^H$ and $\pi\in\Pi^H$, since $\pi^t_h \in\Pi$ for every $(t,h)\in[T]\times[H]$, it holds that 
    \begin{align}
        |\hat{\mathcal{E}}_{\dataTraj}(f; \phat, \rhat) - \mathcal{E}(f; \phat, \rhat) | \leq 2RH \sqrt{\frac{2H\log(8|\mathcal{F}||\Pi|/\delta)}{N}} \leq 4RH \sqrt{\frac{HT\log(8|\mathcal{F}|/\delta)}{N}} \label{eqn:closeness of loss 2}
    \end{align}
    for every $f\in\mathcal{F}$, with probability at least $1-\delta/8$.

    On the other hand, the second term in \eqref{eqn:closeness of loss 1} is bounded by
    \begin{align*}
        &|\mathcal{E}(f; \phat, \rhat) - \mathcal{E}(f; \pstar, \rhat) | \\
        &\leq   \mathbb{E}_{(\tau^0,\tau^1)\sim\piref}\left[ \left| \sum^H_{h=1} (\pstar - \phat)(f_h\circ \pi^t_h)(s^0_h,a^0_h) - \sum^H_{h=1} (\pstar - \phat)(f_h\circ \pi^t_h)(s^1_h,a^1_h) \right|  \right] \\
        &\leq \mathbb{E}_{\tau^0\sim\piref}\left[ \sum^H_{h=1} \left|(\pstar - \phat)(f_h\circ \pi^t_h)(s^0_h,a^0_h)\right| \right] + \mathbb{E}_{\tau^1\sim\piref}\left[ \sum^H_{h=1} \left|(\pstar - \phat)(f_h\circ \pi^t_h)(s^1_h,a^1_h) \right|  \right] \\
        &= 2 \mathbb{E}_{\tau\sim\piref}\left[ \sum^H_{h=1} \left|(\pstar - \phat)(f_h\circ \pi^t_h)(s_h,a_h)\right| \right] \\
        &\leq 2R \mathbb{E}_{\tau\sim\piref}\left[ \sum^H_{h=1} \norm{\pstar(\cdot\mid s_h,a_h) - \phat(\cdot\mid s_h,a_h)}{1}\right] \\
        &= 2R \sum^H_{h=1} \mathbb{E}_{(s_h,a_h)\sim d^{\piref}_h}\left[  \norm{\pstar(\cdot\mid s_h,a_h) - \phat(\cdot\mid s_h,a_h)}{1}\right]
    \end{align*}
    where the first inequality holds since we have $||a|-|b|| \leq |a-b|$ for all $a,b\in\Real$, and the third inequality holds due to H\"older's inequality with the fact that $\norm{f_h \circ \pi^t_h}{\infty} \leq R$. Furthermore, Lemma~\ref{lemma:transition MLE error} implies
    \begin{align}
        |\mathcal{E}(f; \phat, \rhat) - \mathcal{E}(f; \pstar, \rhat) |
        &\leq 2R \sum^H_{h=1} \mathbb{E}_{(s_h,a_h)\sim d^{\piref}_h}\left[  \norm{\pstar(\cdot\mid s_h,a_h) - \phat(\cdot\mid s_h,a_h)}{1}\right] \notag \\
        &\leq 2R \sum^H_{h=1} \sqrt{\mathbb{E}_{(s_h,a_h)\sim d^{\piref}_h}\left[  \norm{\pstar(\cdot\mid s_h,a_h) - \phat(\cdot\mid s_h,a_h)}{1}^2\right]} \notag \\
        &\leq 2RH \epsilon_P(\delta/8) \label{eqn:closeness of loss 3}
    \end{align}
    with probability at least $1-\delta/8$. Taking union bound of the two event \eqref{eqn:closeness of loss 2} and \eqref{eqn:closeness of loss 3}, with probability at least $1-\delta/4$, it holds that
    \begin{align*}
        |\hat{\mathcal{E}}_{\dataTraj}(f; \phat, \rhat) - \mathcal{E}(f; \pstar, \rhat) | 
        &\leq 2RH \sqrt{\frac{HT\log(8|\mathcal{F}|/\delta)}{N}} + 2RH \epsilon_P(\delta/8)
    \end{align*}
    for every $f\in\mathcal{F}$.

    Finally, we conclude the proof by combining the bounds in Step 1 and Step 2. With probability at least $1-\delta/2$, for every $f\in\mathcal{F}$, it hols that
    \begin{align*}
        &\left| \emplossOpt^t(f; \phat, \rhat) - \lossOpt^t(f; \rhat) \right| \\
        &\leq \left|\sum^H_{h=1} \mathbb{E}_{(s_h,a_h)\sim\dataTraj}\left[ f_h\circ\pi^t_h (s_h) - f_h(s_h,a_h)\right] - \sum^H_{h=1} \mathbb{E}_{(s_h,a_h)\sim d^{\piref}_h}\left[ f_h\circ\pi^t_h(s_h) - f_h(s_h,a_h) \right] \right| \\
        &\quad + \left|\hat{\mathcal{E}}_{\dataTraj}(f; \phat, \rhat) - \mathcal{E}(f; \pstar, \rhat) \right| \\
        &\leq 4RH\sqrt{\frac{T\log(8H|\mathcal{F}|/\delta)}{N}} + 2RH \sqrt{\frac{HT\log(8|\mathcal{F}|/\delta)}{N}} + 2RH \epsilon_P(\delta/8) \\
        &\leq 8R\sqrt{\frac{H^3T\log(8H|\mathcal{F}|/\delta)}{N}} + 2RH \epsilon_P(\delta/8).
    \end{align*}
    
\end{proof}

\subsection{Policy Update}

The analysis of the policy update step in Algorithm~\ref{alg:APPO} follows the same argument in Lemma~\ref{lemma:NPG regret, rollout}.

\begin{lemma} \label{lemma:NPG regret}
    For any sequence of functions $\{f^t\}^T_{t=1}$, the policy update (Line 7) in Algorithm~\ref{alg:APPO} guarantees that 
    \begin{align*}
        \frac{1}{T} \sum^T_{t=1} \left( \V{\pistar}{1, r^t}(s_1) - \V{\pi^t}{1, r^t}(s_1) \right) \leq RH\sqrt{\frac{\log|\Aspace|}{2T}}
    \end{align*}
    where $r^t = r^{\pi^t}_{\pstar, f^t}$, i.e. $r^t_h(s,a) = f^t_h(s,a) - \pstar (f^t_{h+1} \circ \pi^t_{h+1})(s,a)$ for all $h\in[H]$ and $(s,a)\in\Sspace\times\Aspace$.
\end{lemma}

\begin{proof}
    Since we have the Bellman equation $f^t_h = r^t_h + \pstar_h (f^t_{h+1}\circ \pi^t_{h+1})$ for all $h\in[H]$, we can apply the performance difference lemma (Lemma~\ref{lemma:performance difference lemma}) to obtain
    \begin{align*}
        \sum^T_{t=1} \left( \V{\pistar}{1, r^t}(s_1) - \V{\pi^t}{1, r^t}(s_1) \right)
        =\sum^T_{t=1} \sum^H_{h=1} \mathbb{E}_{\pistar}\left[  \langle f^t_h(s_h,\cdot), \pistar_h(\cdot\mid s_h) - \pi^t_h(\cdot\mid s_h) \rangle \right].
    \end{align*}
    Rearranging the inner product term, we see that
    \begin{align}
        &\langle \eta f^t_h(s_h,\cdot), \pistar_h(\cdot\mid s) - \pi^t_h(\cdot\mid s) \rangle \notag \\
        &\langle \eta f^t_h(s_h,\cdot), \pistar_h(\cdot\mid s) - \pi^{t+1}_h(\cdot\mid s) \rangle + \langle \eta f^t_h(s_h,\cdot), \pi^{t+1}_h(\cdot\mid s) - \pi^{t}_h(\cdot\mid s) \rangle \notag \\
        &\leq \langle \eta f^t_h(s_h,\cdot), \pistar_h(\cdot\mid s) - \pi^{t+1}_h(\cdot\mid s) \rangle + \eta \norm{f^t_h(s_h,\cdot)}{\infty} \norm{\pistar_h(\cdot\mid s) - \pi^{t+1}_h(\cdot\mid s)}{1} \notag \\
        &\leq \langle \eta f^t_h(s_h,\cdot), \pistar_h(\cdot\mid s) - \pi^{t+1}_h(\cdot\mid s) \rangle + \eta R \norm{\pistar_h(\cdot\mid s) - \pi^{t+1}_h(\cdot\mid s)}{1} \label{eqn:NPG error 1}
    \end{align}
    where we use H\"older's inequality with the fact that $\norm{f^t_h}{\infty}\leq R$. Now recall that the policy update step in Algorithm~\ref{alg:PE oracle} leads to
    \begin{align*}
        \pi^{t+1}_h(\cdot\mid s) = \frac{1}{Z^t_h(s)} \pi^t_h(\cdot\mid s) \exp\left( \eta f^t_h(s,\cdot) \right)
    \end{align*}
    where $Z^t_h(s) = \sum_{a\in\Aspace} \pi^t_h(a \mid s) \exp\left( \eta f^t_h(s,a) \right)$. Using the relationship $\eta f^t_h(s,a) = \log Z^t_h(s) + \log \pi^{t+1}_h(a\mid s) - \log \pi^{t}_h(a\mid s)$, it holds that
    
    \begin{align*}
        &\langle \eta f^t_h(s_h,\cdot), \pistar_h(\cdot\mid s) - \pi^{t+1}_h(\cdot\mid s) \rangle \\
        &= \langle \log Z^t_h(s) + \log \pi^{t+1}_h(\cdot\mid s) - \log \pi^{t}_h(\cdot\mid s), \pistar_h(\cdot\mid s) - \pi^{t+1}_h(\cdot\mid s) \rangle \\
        &= \langle  \log \pi^{t+1}_h(\cdot\mid s) - \log \pi^{t}_h(\cdot\mid s), \pistar(\cdot\mid s) - \pi^{t+1}_h(\cdot\mid s) \rangle \\
        &= \langle  \log \pi^{t+1}_h(\cdot\mid s) - \log \pi^{t}_h(\cdot\mid s), \pistar(\cdot\mid s) \rangle - \kldiv{\pi^{t+1}_h(\cdot\mid s)}{\pi^{t}_h(\cdot\mid s)} \\
        &= \langle  \log \frac{\pi^{t+1}_h(\cdot\mid s)}{\pistar_h(\cdot\mid s)} +  \log \frac{\pistar_h(\cdot\mid s)}{\pi^{t}_h(\cdot\mid s)}, \pistar_h(\cdot\mid s) \rangle - \kldiv{\pi^{t+1}_h(\cdot\mid s)}{\pi^{t}_h(\cdot\mid s)} \\
        &= \kldiv{\pistar_h(\cdot\mid s)}{\pi^t_h(\cdot\mid s)} - \kldiv{\pistar_h(\cdot\mid s)}{\pi^{t+1}_h(\cdot\mid s)} - \kldiv{\pi^{t+1}_h(\cdot\mid s)}{\pi^{t}_h(\cdot\mid s)} \\
        &\leq \kldiv{\pistar_h(\cdot\mid s)}{\pi^t_h(\cdot\mid s)} - \kldiv{\pistar_h(\cdot\mid s)}{\pi^{t+1}_h(\cdot\mid s)} - \frac{1}{2}\norm{\pistar_h(\cdot\mid s) - \pi^{t+1}_h(\cdot\mid s)}{1}^2
    \end{align*}
    where the second equality holds since $Z^t_h(s)$ is a constant given $s$, and the last inequality holds due to Pinsker's inequality. Combining this bound with \eqref{eqn:NPG error 1}, we obtain
    
    \begin{align*}
        & \sum^T_{t=1} \langle \eta f^t_h(s_h,\cdot), \pistar_h(\cdot\mid s) - \pi^t_h(\cdot\mid s) \rangle \\
        &= \sum^T_{t=1} \left( \kldiv{\pistar_h(\cdot\mid s)}{\pi^t_h(\cdot\mid s)} - \kldiv{\pistar_h(\cdot\mid s)}{\pi^{t+1}_h(\cdot\mid s)} \right) \\
        &\quad + \sum^T_{t=1} \left( \eta R \norm{\pistar_h(\cdot\mid s) - \pi^{t+1}_h(\cdot\mid s)}{1} - \frac{1}{2}\norm{\pistar_h(\cdot\mid s) - \pi^{t+1}_h(\cdot\mid s)}{1}^2 \right) \\
        &\leq \sum^T_{t=1} \left( \kldiv{\pistar_h(\cdot\mid s)}{\pi^t_h(\cdot\mid s)} - \kldiv{\pistar_h(\cdot\mid s)}{\pi^{t+1}_h(\cdot\mid s)} \right) + \sum^T_{t=1} \frac{\eta^2 R^2}{2} \\
        &= \kldiv{\pistar_h(\cdot\mid s)}{\pi^1_h(\cdot\mid s)} - \kldiv{\pistar_h(\cdot\mid s)}{\pi^{T+1}_h(\cdot\mid s)} + \frac{\eta^2 R^2 T}{2} \\
        &\leq \log|\Aspace| + \frac{\eta^2 R^2 T}{2}
    \end{align*}
    where the first inequality holds since $\forall x\in\Real \, ax - x^2/2 \leq a^2/2$, and the second inequality holds due to the fact that $\pi^1_h = \text{Unif}(\Aspace)$. Finally, setting $\eta = \sqrt{\frac{2\log|\Aspace|}{R^2 T}}$, we have
    
    \begin{align*}
         \sum^T_{t=1} \V{\pistar}{1, r^t} - \V{\pi^t}{1, r^t}
        &=\sum^T_{t=1}  \mathbb{E}_{\pistar}\left[ \sum^H_{h=1} \langle f^t_h(s_h,\cdot), \pistar_h(\cdot\mid s_h) - \pi^t_h(\cdot\mid s_h) \rangle \right] \\
        &=\sum^H_{h=1}  \mathbb{E}_{\pistar}\left[ \sum^T_{t=1} \langle f^t_h(s_h,\cdot), \pistar_h(\cdot\mid s_h) - \pi^t_h(\cdot\mid s_h) \rangle \right] \\
        &\leq \sum^H_{h=1} \left( \frac{\log|\Aspace|}{\eta} + \frac{\eta R^2 T}{2} \right) = RH\sqrt{\frac{T\log|\Aspace|}{2}}.
    \end{align*}

\end{proof}

Finally, we prove Theorem~\ref{thm:upper bound}.

\begin{proof}[Proof of Theorem~\ref{thm:upper bound}]
    For simplicity, we write $r^t = r^{\pi^t}_{\pstar, f^t}$, i.e. $r^t_h(s,a) = f^t_h(s,a) - \pstar_h (f^t_{h+1} \circ \pi^t_{h+1})(s,a)$ for all $(s,a)\in\Sspace\times\Aspace$ and $h\in[H]$. The condition $r^t\in\Rset^H$ is not required; we only rely on the boundedness $\norm{r^t_h}{\infty}\leq R$ for all $h$, which Assumption~\ref{assumption:value function class} guarantees.
    
    Condition on the events in Lemma~\ref{lemma:reward MLE error} (with $\delta'=\delta/2$) and Lemma~\ref{lemma:optimization error}, which hold simultaneously with probability at least $1-\delta$. Consider the following sub-optimality decomposition at step $t$:
    \begin{align}
        \V{\pistar}{1, \rstar} - \V{\pi^t}{1, \rstar} \notag
        &= \V{\pistar}{1, \rstar} - \V{\pistar}{1, \rhat} + \V{\pistar}{1, \rhat} - \V{\pistar}{1, r^t} + \V{\pistar}{1, r^t} - \V{\pi^t}{1, \rstar} + \V{\pi^t}{1, r^t} - \V{\pi^t}{1, r^t} \notag \\
        &= \underbrace{\V{\pistar}{1,\rstar-\rhat} - \V{\piref}{1,\rstar-\rhat}}_{\text{(I) : MLE estimation error}} \notag \\
        &\quad + \underbrace{ \V{\pistar}{1,\rhat-r^t} - \V{\piref}{1,\rhat-r^t} - \V{\pi^t}{1,\rstar} + \V{\piref}{1, \rstar} + \V{\pi^t}{1,r^t} - \V{\piref}{1, r^t}}_{\text{(II) : Optimization error}} \notag \\
        &\quad + \underbrace{\V{\pistar}{1, r^t} - \V{\pi^t}{1, r^t} }_{\text{(III) : Policy update regret}}, \label{eqn:upper bound 1}
    \end{align}
    where we omit the initial state $s_1$ for simplicity. 

    \paragraph{Bounding (I).}
    Using Lemma~\ref{lemma:reward MLE error}, the MLE estimation error is bounded by:
    \begin{align*}
        \text{(I)} &= \V{\pistar}{1,\rstar-\rhat} - \V{\piref}{1,\rstar-\rhat} \\
        &= \mathbb{E}_{\tau^0\sim \pistar, \tau^1 \sim \piref} \left[ \rstar(\tau^0) - \rstar(\tau^1) - \rhat(\tau^0) + \rhat(\tau^1)  \right] \\
        &\leq \sqrt{\mathbb{E}_{\tau^0\sim \pistar, \tau^1 \sim \piref} \left[ | \rstar(\tau^0) - \rstar(\tau^1) - \rhat(\tau^0) + \rhat(\tau^1) |^2  \right]} \\
        &\leq \sqrt{\Ctr \mathbb{E}_{\tau^0, \tau^1 \sim \piref} \left[ | \rstar(\tau^0) - \rstar(\tau^1) - \rhat(\tau^0) + \rhat(\tau^1) |^2  \right]} \\
        &\leq \sqrt{\Ctr} \epsilon_{r}(\delta/2)
    \end{align*}

    \paragraph{Bounding (II).}
    We can relate the terms $\V{\pistar}{1,\rhat-r^t} - \V{\piref}{1,\rhat-r^t}$ to the trajectory-pair $\ell_1$ loss $\mathcal{E}(f^t; \pstar, \rhat)$. By Assumption~\ref{assumption:concentrability}, we have that
    \begin{align*}
       &\V{\pistar}{1,\rhat-r^t} - \V{\piref}{1,\rhat-r^t}  \\
        &= \mathbb{E}_{\tau^0\sim \pistar, \tau^1 \sim \piref} \left[ \rhat(\tau^0) - \rhat(\tau^1) - r^t(\tau^0) + r^t(\tau^1)  \right]  \\
        &\leq \Ctr \mathbb{E}_{\tau^0, \tau^1 \sim \piref} \left[ | \rhat(\tau^0) - \rhat(\tau^1) - r^t(\tau^0) + r^t(\tau^1) |  \right] \\
        &= \Ctr \mathbb{E}_{\tau^0, \tau^1 \sim \piref} \left[ | r^{\pi^t}_{\pstar,f^t}(\tau^0) - r^{\pi^t}_{\pstar,f^t}(\tau^1) - \rhat(\tau^0) + \rhat(\tau^1) |  \right] \\
        &= \Ctr \mathcal{E}(f^t; \pstar, \rhat) \leq \lambda \mathcal{E}(f^t; \pstar, \rhat)
    \end{align*}
    where the last inequality holds since $\mathcal{E}(f^t; \pstar, \rhat)$ is non-negative and $\lambda \geq \Ctr$. Further, Lemma~\ref{lemma:optimization error} and Lemma~\ref{lemma:performance difference lemma} implies
    \begin{align*}
        \lambda \mathcal{E}(f^t; \pstar, \rhat) 
        &\leq \sum^H_{h=1} \mathbb{E}_{(s_h,a_h)\sim d^{\piref}_h}\left[ \Q{\pi^t}{h}\circ\pi^t_h (s_h) - \Q{\pi^t}{h}(s_h, a_h) \right]  + \lambda \mathcal{E}(\Q{\pi^t}{}; \pstar, \rhat) \\
        &\quad - \sum^H_{h=1} \mathbb{E}_{(s_h,a_h)\sim d^{\piref}_h}\left[ f^t_h\circ\pi^t_h (s_h) - f^t_h(s_h, a_h) \right] + 2\epsilon_{approx}\\
        &= \left( \V{\pi^t}{1,\rstar} - \V{\piref}{1, \rstar} \right) + \lambda \mathcal{E}(\Q{\pi^t}{}; \pstar, \rhat) -\left( \V{\pi^t}{1,r^t} - \V{\piref}{1, r^t} \right) + 2\epsilon_{approx}.
    \end{align*}

    On the other hand, note that
    \begin{align*}
        r^{\pi^t}_{\pstar, \Q{\pi^t}{}}(\tau) &= \sum^H_{h=1} \left( \Q{\pi^t}{h}(s_h,a_h) - \pstar (\Q{\pi^t}{h+1}\circ \pi^t_{h+1})(s_h,a_h) \right) \\
        &= \sum^H_{h=1} \left( \Q{\pi^t}{h}(s_h,a_h) - \pstar\V{\pi^t}{h+1}(s_h,a_h) \right) \\
        &= \sum^H_{h=1} \rstar_h(s_h,a_h) = \rstar(\tau)
    \end{align*}
    for any $\tau = (s_1,a_1,\dots,s_H,a_H)$, i.e. $r^{\pi^t}_{\pstar, \Q{\pi^t}{}} = \rstar$. Thus, we have
    \begin{align*}
        \lambda \mathcal{E}(\Q{\pi^t}{}; \pstar, \rhat) &=
        \lambda \mathbb{E}_{(\tau_0, \tau_1) \sim \piref}\left[ \left| \{\rstar(\tau^0) - \rstar(\tau^1)\} - \{\rhat(\tau^0) - \rhat(\tau^1)\} \right| \right] \\
        &\leq \lambda \sqrt{ \mathbb{E}_{(\tau_0, \tau_1) \sim \piref}\left[ \left| \{\rstar(\tau^0) - \rstar(\tau^1)\} - \{\rhat(\tau^0) - \rhat(\tau^1)\} \right|^2 \right] } \leq \lambda \epsilon_r (\delta/2)
    \end{align*}
    where the inequality follows from Lemma~\ref{lemma:reward MLE error}. Combining the results, we obtain
    \begin{align*}
        \text{(II)} &= \left(\V{\pistar}{1,\rhat-r^t} - \V{\piref}{1,\rhat-r^t} \right) - \left( \V{\pi^t}{1,\rstar}  - \V{\piref}{1, \rstar}\right) + \left(\V{\pi^t}{1,r^t} - \V{\piref}{1, r^t} \right) \\
        &\leq \lambda \mathcal{E}(\Q{\pi^t}{}; \pstar, \rhat) + 2\epsilon_{approx} \leq \lambda \epsilon_r(\delta/2) + 2\epsilon_{approx}
    \end{align*}

    \paragraph{Bounding Sub-optimality.}
    Finally, we bound the sub-optimality $\V{\pistar}{1, \rstar} - \V{\hat{\pi}}{1, \rstar}$.
    Putting the bounds on (I) and (II) into \eqref{eqn:upper bound 1}, we have
    \begin{align*}
        &\V{\pistar}{1, \rstar} - \V{\pi^t}{1, \rstar} \notag \\
        &\leq \sqrt{\Ctr} \epsilon_{r}(\delta/2) + \lambda \epsilon_r(\delta/2) + 2\epsilon_{approx} + \V{\pistar}{1, r^t} - \V{\pi^t}{1, r^t}
    \end{align*}
    
    Since Algorithm~\ref{alg:APPO} returns the mixture policy $\hat{\pi} = \frac{1}{T}\sum^T_{t=1} \pi^t$, the sub-optimality is $\V{\pistar}{1, \rstar} - \V{\hat{\pi}}{1, \rstar} = \frac{1}{T} \sum^T_{t=1} \left( \V{\pistar}{1, \rstar} - \V{\pi^t}{1, \rstar} \right)$. Using the bounds we derived and Lemma~\ref{lemma:NPG regret}, it holds that
    
    \begin{align*}
        &\V{\pistar}{1, \rstar} - \V{\hat{\pi}}{1, \rstar} \\
        &= \frac{1}{T} \sum^T_{t=1} \left( \V{\pistar}{1, \rstar} - \V{\pi^t}{1, \rstar} \right) \\
        &\leq \sqrt{\Ctr} \epsilon_{r}(\delta/2) + \lambda \epsilon_r(
        \delta/2
        ) + 2\epsilon_{approx} + \frac{1}{T} \sum^T_{t=1} \left( \V{\pistar}{1, r^t} - \V{\pi^t}{1, r^t} \right) \\
        &\leq \sqrt{\Ctr} \epsilon_{r}(\delta/2) + \lambda \epsilon_r(\delta/2) + 2\epsilon_{approx} + RH\sqrt{\frac{\log|\Aspace|}{2T}} \\
        &\leq \mathcal{O}\left( \Ctr \sqrt{\frac{ \kappa^2 H \log(|\Rset|/\delta)}{M}} + R \sqrt{\frac{H^3T\log(H|\mathcal{F}|/\delta)}{N}} + RH\sqrt{\frac{ \log(H|\Pset|/\delta)}{N}} + RH\sqrt{\frac{\log|\Aspace|}{T}} \right) \\
        &\leq \mathcal{O}\left( \Ctr \sqrt{\frac{ \kappa^2 H \log(|\Rset|/\delta)}{M}} + RH \sqrt{\frac{\max\{HT\log(H|\mathcal{F}|/\delta),  \log(H|\Pset|/\delta)\}}{N}} + RH\sqrt{\frac{\log|\Aspace|}{T}} \right).
    \end{align*}
    
\end{proof}

\section{Supporting Lemmas}

\begin{lemma}[Performance Difference Lemma~\citep{kakade2002approximately}] \label{lemma:performance difference lemma}
    Let $P$ be any transition probability, and denote the corresponding value function by $V$. Let $\pi, \tilde{pi}$ be any policies. For any reward $r$, we have that
    \begin{align*}
        \V{\pi}{1, r}(s_1) - \V{\tilde{\pi}}{1, r}(s_1) = \sum^H_{h=1}\mathbb{E}_{s_h\sim d^{\pi}_h}\left[ \langle \Q{\tilde{\pi}}{h, r}(s_h,\cdot), \pi(\cdot\mid s_h) - \tilde{\pi}(\cdot\mid s_h) \rangle \right]
    \end{align*}
\end{lemma}

\begin{proof}
    Recursively applying the Bellman equation, we obtain
    \begin{align*}
        \V{\pi}{1, r}(s_1) - \V{\tilde{\pi}}{1, r}(s_1) &= \E_{\pi}[r(s_1,a_1) + \V{\pi}{2,r}(s_2)] - \E_{\pi}[\V{\tilde{\pi}}{1,r}(s_1)] \\
        &=\E_{\pi}[\Q{\tilde{\pi}}{1, r}(s_1, a_1) - \V{\tilde{\pi}}{2, r}(s_2) + \V{\pi}{2,r}(s_2)] - \E_{\pi}[\V{\tilde{\pi}}{1,r}(s_1)] \\
        &= \E_{\pi}[\Q{\tilde{\pi}}{1, r}(s_1, a_1) - \V{\tilde{\pi}}{1,r}(s_1)] + \E_{\pi}[\V{\pi}{2,r}(s_2) - \V{\tilde{\pi}}{2, r}(s_2)] \\
        &=\E_{\pi}[ \langle \Q{\tilde{\pi}}{1, r}(s_1, \cdot), \pi(\cdot\mid s_1) - \tilde{\pi}(\cdot\mid s_1) \rangle ] + \E_{\pi}[\V{\pi}{2,r}(s_2) - \V{\tilde{\pi}}{2, r}(s_2)] \\
        &= \cdots \\
        &= \sum^H_{h=1}\mathbb{E}_{s_h\sim d^{\pi}_h}\left[ \langle \Q{\tilde{\pi}}{h, r}(s_h,\cdot), \pi(\cdot\mid s_h) - \tilde{\pi}(\cdot\mid s_h) \rangle \right].
    \end{align*}
\end{proof}


\begin{lemma}[Lemma 2 in \cite{zhan2024provable}]\label{lemma:reward MLE error}
    With probability at least $1-\delta'$, we have
    \begin{align*}
        \mathbb{E}_{\tau^0,\tau^1\sim \piref}\left[ |(\rhat(\tau^0)-\rhat(\tau^1)) - (\rstar(\tau^0)-\rstar(\tau^1))|^2 \right] \leq \frac{c_1 \kappa^2 H \log(|\Rset|/\delta')}{M} := \epsilon_{r}^2(\delta')
    \end{align*}
\end{lemma}

\begin{lemma}[Lemma 3 in \cite{zhan2024provable}]\label{lemma:transition MLE error}
    With probability at least $1-\delta'$, for all $h\in[H]$, it holds that
    \begin{align*}
        \mathbb{E}_{(s_h,a_h)\sim d^{\piref}_h}\left[ \norm{\phat_h(\cdot\mid s_h,a_h) - \pstar(\cdot\mid s_h,a_h)}{1}^2 \right] \leq \frac{c_2 \log(H|\Pset|/\delta')}{N} := \epsilon_{P}^2(\delta')
    \end{align*}
    where $c_2$ is an absolute constant.
\end{lemma}

\begin{lemma}[Lemma 15 in \citet{song2023hybrid}]\label{lemma:least square}
    Fix any $B>0$, $\delta\in (0,1)$ and assume we have a class of real-valued functions $\mathcal{H} : \mathcal{X} \rightarrow [-B,B]$. Suppose we have $K$ i.i.d. samples $\{(x_k,y_k)\}^K_{k=1}$ where $x_k \sim \rho$ and $y_k = h^\star(x_k) + \epsilon_k$ where $h^\star \in \mathcal{H}$ and $\{\epsilon_k\}^K_{k=1}$ are independent random variables such that $\mathbb{E}[\epsilon_k \mid x_k] = 0$. Additionally, suppose that $\max_{k} |y_k| \leq R$ and $\sup_{x\in\mathcal{X}} |h^\star(x)| \leq B$. Then, with probability at least $1-\delta$, the least square estimator $\hat{h} \in \argmin_{h\in\mathcal{H}} \sum^K_{k=1} (h(x_k)-y_k)^2$ satisfies:
    \begin{align*}
        \mathbb{E}_{x\sim\rho}\left[ \left( \hat{h}(x) - h^\star(x) \right)^2 \right] \leq \frac{c_2 B^2 \log(|\mathcal{H}|/\delta)}{K}
    \end{align*}
    where $c_2$ is an absolute constant.
\end{lemma}

\begin{algorithm}[h]
    \caption{\texttt{APPO} (Practical version)} \label{alg:practical}
    \begin{algorithmic}[1]
        \State \textbf{Input:} Batch size $B$, Learning rates $\alpha_\phi, \alpha_\psi, \alpha_\theta$, constants $\lambda>0$, $\tau\in(0,1)$
        \State Train reward model $\rhat$ based on $\dataPref$ \Comment{Use any reward learning method}
        \For{\texttt{step}$=1,2,\dots$}
            \State Sample mini-batch of transition tuples $\mathcal{B}_{\text{tup}}$ and trajectory pairs $\mathcal{B}_{\text{traj}}$ from $\dataTraj$
            \State Train Q functions $\phi^i \leftarrow \phi^i - \alpha_{\phi} \nabla_{\phi^i} \mathcal{L}^{\lambda}_{\phi^i}(\mathcal{B}_{\text{tup}}, \mathcal{B}_{\text{traj}})$ for $i\in\{1,2\}$ \eqref{eqn:practical q loss}
            \State Update target Q function $\bar{\phi}^i = (1-\tau)\bar{\phi}^i + \tau\phi^i$ for $i\in\{1,2\}$
            \State Train V function $\psi \leftarrow \psi - \alpha_{\psi} \nabla_{\psi} \mathcal{L}_{\psi}(\mathcal{B}_{\text{tup}})$ \eqref{eqn:practical v loss}
            \State Train actor $\theta \leftarrow \theta + \alpha_{\theta} \nabla_\theta \mathcal{L}_{\theta}(\mathcal{B}_{\text{tup}})$ \eqref{eqn:practical actor loss}
        \EndFor
    \end{algorithmic}
\end{algorithm}

\newpage

\section{Additional Experiments} \label{sec:additional experiments}

\subsection{Evaluation on the Meta-World \texttt{medium-expert} Dataset}

To further assess the generalization capability of \texttt{APPO}, we collected the Meta-World \texttt{medium-expert} dataset following the data collection procedures outlined in prior works \citep{hejna2024inverse,choi2024listwise}. Detailed information on the dataset is provided in Section~\ref{sec:experimental details}. For comparison, we use MR, the most effective baseline method identified in Table~\ref{tab:main experiment}. The results in Table~\ref{tab:medium-expert} show that \texttt{APPO} consistently outperforms or matches MR.

\begin{table}[h]
\centering
\begin{tabular}{@{}c|cc|cc@{}}
\toprule
\# of feedback & \multicolumn{2}{c|}{500}             & \multicolumn{2}{c}{1000}             \\ \midrule
Dataset        & dial-turn & sweep-into  & dial-turn & sweep-into  \\ \midrule
MR &$15.80$\tiny$\pm12.73$ &$14.32$\tiny$\pm3.39$ &$26.08$\tiny$\pm18.78$ &$8.48$\tiny$\pm1.92$ \\
\texttt{APPO} &$32.40$\tiny$\pm13.56$ &$12.80$\tiny$\pm5.35$ &$39.20$\tiny$\pm15.69$ &$14.56$\tiny$\pm6.25$  \\ \bottomrule
\end{tabular}
\caption{Success rates on Meta-World \texttt{medium-expert} dataset with $500$, $1000$ preference feedback samples, averaged over $5$ random seeds.}
\label{tab:medium-expert}
\end{table}

\subsection{Effect of Dataset Size}

\begin{figure}[!h]
    \centering
    \includegraphics[width=0.8\linewidth]{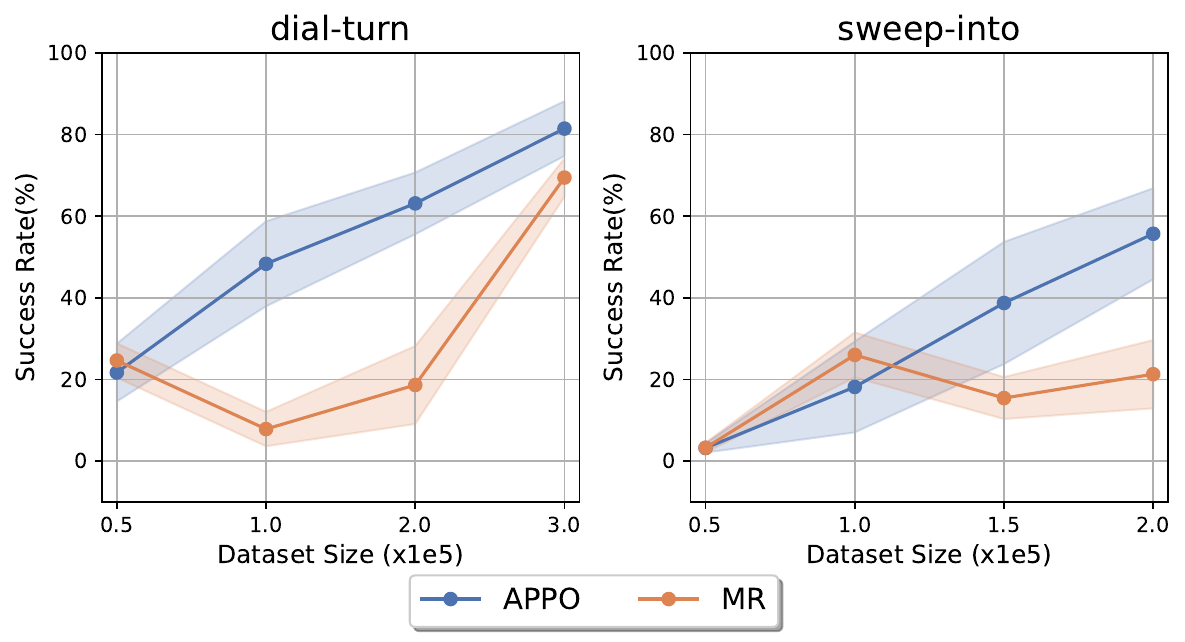}
    \caption{Success rates of \texttt{APPO} and MR evaluated in Meta-World \texttt{medium-replay} datasets, with varying dataset sizes. The number of preference feedback samples is fixed at $1000$.}
    \label{fig:data quality}
\end{figure}

To examine the impact of dataset size $|\dataTraj|$, we conducted experiments with varying sizes of the Meta-World \texttt{medium-replay} datasets. As shown in Figure~\ref{fig:data quality}, the performance of MR fluctuates with changes in dataset size, whereas the performance of \texttt{APPO} exhibits a more consistent and gradual response to dataset size variations.

\subsection{Learning Curves from Experiments.}

Figure~\ref{fig:learning curve} and Figure~\ref{fig:medium-expert} show the learning curves of the experiments in Table~\ref{tab:main experiment} and Table~\ref{tab:medium-expert}. Each algorithm is trained for 250,000 gradient steps, with evaluations conducted every 5,000 steps. The success rates from the final five evaluation points are averaged and reported in Table~\ref{tab:main experiment} and Table~\ref{tab:medium-expert}.

\section{Experimental Details} \label{sec:experimental details}

\subsection{Datasets}

\begin{table}[!h]
\centering
\resizebox{\textwidth}{!}{%
\begin{tabular}{@{}l|llllllll@{}}
\toprule
Dataset & BPT & box-close & dial-turn & sweep & BPT-wall & sweep-into & drawer-open & lever-pull \\ \midrule
Size ($\times 10^5$)  & $1.0$ & $8.0$       & $3.0$       & $7.0$   & $1.5$      & $1.0$        & $1.0$         & $3.0$        \\ \bottomrule
\end{tabular}%
}
\caption{The sizes of Meta-World \texttt{medium-replay} datasets \citep{choi2024listwise}. The abbreviation BPT indicates button-press-topdown.}
\label{tab:dataset size}
\end{table}

The Meta-World \texttt{medium-replay} dataset from \citet{choi2024listwise} consists of replay buffers generated by SAC~\citep{haarnoja2018soft} agents with an approximate success rate of $50$\%. The dataset sizes are detailed in Table~\ref{tab:dataset size}.

The Meta-World \texttt{medium-expert} dataset was collected following the procedures described in prior works~\citep{hejna2024inverse,choi2024listwise}.
Each dataset contains trajectories from five sources: (1) an expert policy, (2) expert policies for randomized variants and goals of the task, (3) expert policies for different tasks, (4) a random policy, and (5) an $\epsilon$-greedy expert policy that takes greedy actions with a $50$\% probability. These trajectories are included in the dataset in proportions of 1 : 1 : 2 : 4 : 4, respectively. Additionally, standard Gaussian noise was added to the actions of each policy. The dataset sizes match those of the \texttt{medium-replay} dataset.

\subsection{Implementation and Hyperparameters.}

For a fair comparison with baseline methods, we train the reward model and MR following the official implementation of \citet{choi2024listwise}. The reward model is implemented by an ensemble model of three fully connected neural networks with three hidden layers, each containing 128 neurons. For critics (Q and V) and policies, we use fully connected neural networks with three hidden layers of 256 neurons each. Other hyperparameters are listed in Table~\ref{tab:hparameters}.
We find that using a lower learning rate for $\pi$ and softer target network updates improves the stability of \texttt{APPO} training. Experiments were conducted on an Intel Xeon Gold 6226R CPU and an Nvidia GeForce RTX 3090 GPU. Each training session consists of 250,000 gradient steps, taking approximately 3-4 hours to complete. Our code is available at \url{https://github.com/oh-lab/APPO.git}.

\begin{table}[h]
\centering
\resizebox{\textwidth}{!}{%
\begin{tabular}{@{}lll@{}}
\toprule
Algorithm                     & Component                            & Value                                                  \\ \midrule
\multirow{6}{*}{Reward model} & Neural networks                            & $3$-layers, hidden dimension $128$            \\
& Activation                            & ReLU for hidden activations, Tanh for final activation \\
& Optimizer                            & Adam~\citep{kingma2015adam} with learning rate 1e-3                           \\
                              & Batch size                           & 512                                                    \\
                              & Epochs                               & 300                                                    \\
                              & Ensembles                            & 3                                                      \\ \midrule
\multirow{8}{*}{MR}           & Neural networks (Q, V, $\pi$)                           & $3$-layers, hidden dimension $256$            \\
& Activaton                            & ReLU for hidden activations \\
& Q, V, $\pi$ optimizer   & Adam with learning rate 3e-4                           \\
                              & Batch size                           & 256                                                    \\
                              & Target network soft update           & 0.005                                                  \\
                              & $\beta$ (IQL advantage weight)     & 3.0                                                    \\
                              & $\tau$ (IQL expectile parameter)  & 0.7                                                    \\
                              & discount factor                   & 0.99                                                   \\ \midrule
\multirow{7}{*}{\texttt{APPO}}         & Neural networks (Q, V, $\pi$)                           & $3$-layers, hidden dimension $256$            \\
& Activaton                            & LeakyReLU for hidden activations \\
& Q,V, $\alpha$ optimizer & Adam with learning rate 3e-4                           \\
                              & $\pi$ optimizer         & Adam with learning rate 3e-5                           \\
                              & Batch size                           & 256 transitions and 16 trajectory pairs                \\
                              & Target network soft update           & 0.001                                                  \\
                              & discount factor                      & 0.99                                                   \\ \bottomrule
\end{tabular}%
}
\caption{Implementation details and hyperparameters. For the reward model and MR algorithm, we follow the official implementation of \citet{choi2024listwise}.}
\label{tab:hparameters}
\end{table}

\begin{figure}[!h]
    \centering
    \includegraphics[width=\linewidth]{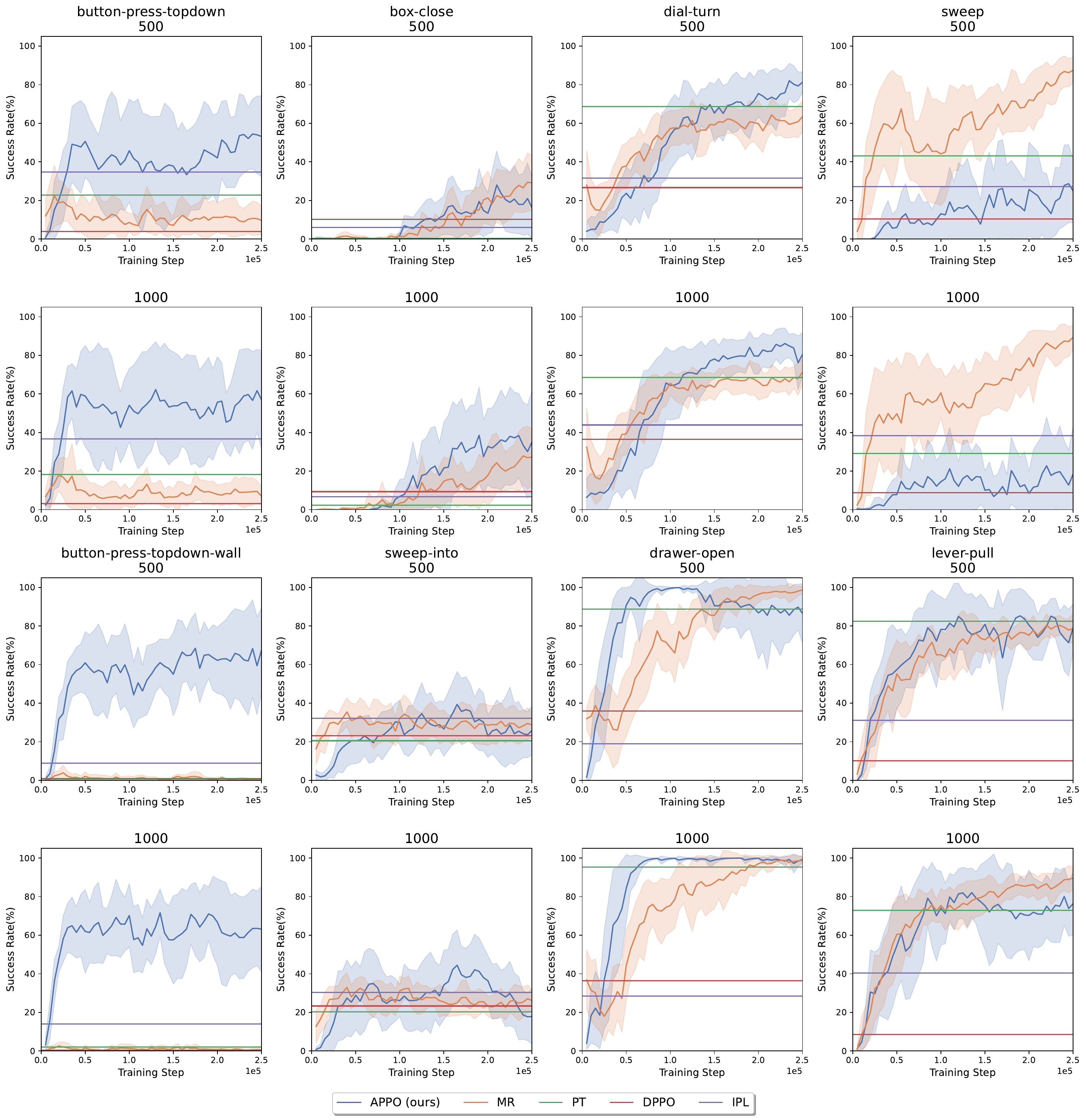}
    \caption{Learning Curves from the experiments in Table~\ref{tab:main experiment}.}
    \label{fig:learning curve}
\end{figure}

\begin{figure}[!h]
    \centering
    \includegraphics[width=0.7\linewidth]{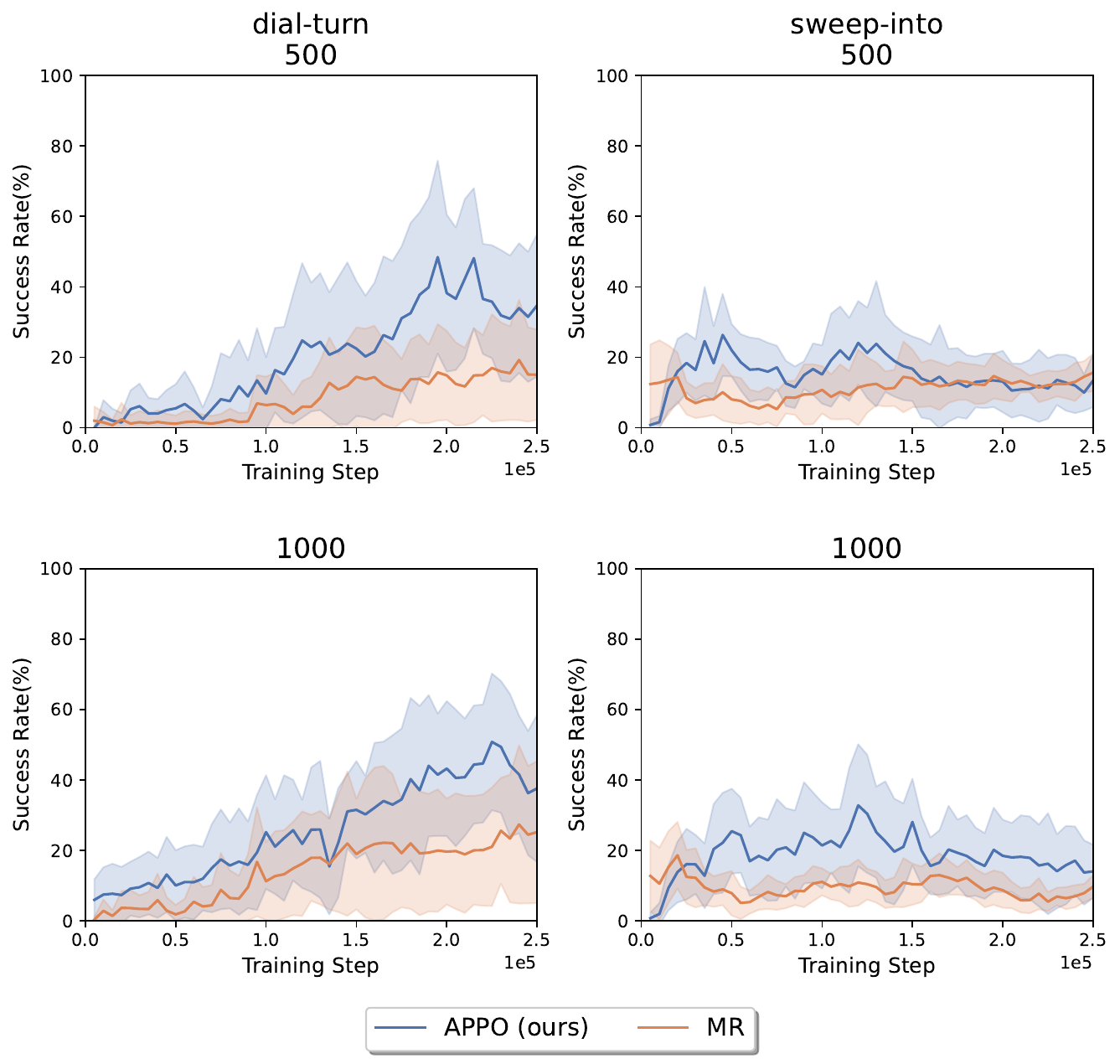}
    \caption{Learning Curves from the experiments in Table~\ref{tab:medium-expert}.}
    \label{fig:medium-expert}
\end{figure}

\end{document}